\documentclass{article}

\usepackage[colorlinks=true, linkcolor=blue, citecolor=blue,urlcolor=black]{hyperref}

\usepackage[nonatbib,preprint]{neurips_2021}

\usepackage[utf8]{inputenc} 
\usepackage[T1]{fontenc}    
\usepackage{hyperref}       
\usepackage{url}            
\usepackage{booktabs}       
\usepackage{amsfonts}       
\usepackage{nicefrac}       
\usepackage{microtype}      
\usepackage{mathtools}
\usepackage{natbib}
\usepackage{xcolor}
\usepackage{bbm} 

\usepackage[ruled]{algorithm2e}
\usepackage{algorithmic}

\usepackage{chngcntr}
\usepackage{apptools}
\AtAppendix{\counterwithin{lemma}{section}}

\usepackage{amsmath}
\usepackage{mathtools}
\usepackage{amsthm}
\usepackage{amssymb}

\newcommand{\calA}{{\mathcal{A}}}

\newcommand{\calX}{{\mathcal{X}}}
\newcommand{\calS}{{\mathcal{S}}}
\newcommand{\calF}{{\mathcal{F}}}

\newcommand{\calK}{{\mathcal{K}}}
\newcommand{\calL}{{\mathcal{L}}}

\newcommand{\calT}{{\mathcal{T}}}

\newcommand{\vtheta}{\nu_{\theta}^{o_{1:t}, a_{1:t}}}
\newcommand{\vthetatau}{\nu_{\theta}^{o_{1:\tau-1}, a_{1:\tau-1}}}
\newcommand{\vgammatau}{\nu_{\gamma}^{o_{1:\tau-1}, a_{1:\tau-1}}}
\newcommand{\vthetastartau}{\nu_{\theta_*}^{o_{1:\tau-1}, a_{1:\tau-1}}}
\newcommand{\vgamma}{\nu_{\gamma}^{o_{1:t}, a_{1:t}}}

\newcommand{\1}{\mathbbm{1}}
\newcommand{\p}{\prime}

\newcommand{\sched}{\texttt{SCHED}}

\DeclareMathOperator*{\argmin}{argmin}
\DeclareMathOperator*{\argmax}{argmax}
\DeclareMathOperator*{\spn}{sp}

\newcommand{\eat}[1]{}

\newcommand{\rbr}[1]{\left(#1\right)}
\newcommand{\sbr}[1]{\left[#1\right]}

\newcommand{\abr}[1]{\left|#1\right|}

\DeclarePairedDelimiter\ceil{\lceil}{\rceil}

\newcommand{\field}[1]{\mathbb{#1}}

\newcommand{\fR}{\field{R}}

\newcommand{\E}{\field{E}}
\newcommand{\Ec}{\field{E_{\theta_*}}}
\newcommand{\pc}{\field{P_{\theta_*}}}
\renewcommand{\p}{\field{P}}

\newtheorem{lemma}{Lemma}
\newtheorem{theorem}{Theorem}

\newtheorem{definition}{Definition}
\newtheorem{assumption}{Assumption}

\newcommand{\order}{\ensuremath{\mathcal{O}}}

\newcommand{\otil}{\ensuremath{\tilde{\mathcal{O}}}}

\usepackage{prettyref}

\newcommand{\savehyperref}[2]{\texorpdfstring{\hyperref[#1]{#2}}{#2}}
\newrefformat{eq}{\savehyperref{#1}{Eq.~\textup{(\ref*{#1})}}}
\newrefformat{eqn}{\savehyperref{#1}{Equation~\ref*{#1}}}
\newrefformat{lem}{\savehyperref{#1}{Lemma~\ref*{#1}}}
\newrefformat{def}{\savehyperref{#1}{Definition~\ref*{#1}}}
\newrefformat{line}{\savehyperref{#1}{Line~\ref*{#1}}}
\newrefformat{thm}{\savehyperref{#1}{Theorem~\ref*{#1}}}
\newrefformat{corr}{\savehyperref{#1}{Corollary~\ref*{#1}}}
\newrefformat{cor}{\savehyperref{#1}{Corollary~\ref*{#1}}}
\newrefformat{sec}{\savehyperref{#1}{Section~\ref*{#1}}}
\newrefformat{subsec}{\savehyperref{#1}{Section~\ref*{#1}}}
\newrefformat{app}{\savehyperref{#1}{Appendix~\ref*{#1}}}
\newrefformat{assum}{\savehyperref{#1}{Assumption~\ref*{#1}}}
\newrefformat{ex}{\savehyperref{#1}{Example~\ref*{#1}}}
\newrefformat{fig}{\savehyperref{#1}{Figure~\ref*{#1}}}
\newrefformat{alg}{\savehyperref{#1}{Algorithm~\ref*{#1}}}
\newrefformat{rem}{\savehyperref{#1}{Remark~\ref*{#1}}}
\newrefformat{conj}{\savehyperref{#1}{Conjecture~\ref*{#1}}}
\newrefformat{prop}{\savehyperref{#1}{Proposition~\ref*{#1}}}
\newrefformat{proto}{\savehyperref{#1}{Protocol~\ref*{#1}}}
\newrefformat{prob}{\savehyperref{#1}{Problem~\ref*{#1}}}
\newrefformat{claim}{\savehyperref{#1}{Claim~\ref*{#1}}}
\newrefformat{que}{\savehyperref{#1}{Question~\ref*{#1}}}
\newrefformat{op}{\savehyperref{#1}{Open Problem~\ref*{#1}}}
\newrefformat{fn}{\savehyperref{#1}{Footnote~\ref*{#1}}}
\newrefformat{tab}{\savehyperref{#1}{Table~\ref*{#1}}}
\newrefformat{fig}{\savehyperref{#1}{Figure~\ref*{#1}}}

\title{Online Learning for Unknown Partially Observable MDPs}

\author{
  Mehdi Jafarnia-Jahromi \\
  University of Southern California\\
  \texttt{mjafarni@usc.edu} \\
  \And
  Rahul Jain \\
  University of Southern California\\
  \texttt{rahul.jain@usc.edu} \\
  \And
  Ashutosh Nayyar \\
  University of Southern California\\
  \texttt{ashutosn@usc.edu} \\
}

\begin{document}

\maketitle

\begin{abstract}
Solving Partially Observable Markov Decision Processes (POMDPs) is hard. Learning optimal controllers for POMDPs when the model is unknown is harder. Online learning of optimal controllers for unknown POMDPs, which requires efficient learning using regret-minimizing algorithms that effectively tradeoff exploration and exploitation, is even harder, and no solution exists currently. In this paper, we consider infinite-horizon average-cost POMDPs with unknown transition model, though a known observation model. We propose a natural posterior sampling-based reinforcement learning algorithm (\texttt{PSRL-POMDP}) and show that it achieves a regret bound of $\order(\log T)$, where $T$ is the time horizon, when the parameter set is finite. 
In the general case (continuous parameter set), we show that the algorithm achieves $\otil (T^{2/3})$ regret under two technical assumptions. To the best of our knowledge, this is the first online RL algorithm for POMDPs and has sub-linear regret.
\end{abstract}


\section{Introduction}
\label{sec: introduction}

Reinforcement learning (RL) considers the sequential decision-making problem of an agent in an unknown environment with the goal of minimizing the total cost. The agent faces a fundamental \textit{exploration-exploitation} trade-off: should it \textit{exploit} the available information to minimize the cost or should it \textit{explore} the environment to gather more information for future decisions? Maintaining a proper balance between exploration and exploitation is a fundamental challenge in RL and is measured with the notion of cumulative regret: the difference between the cumulative cost of the learning algorithm and that of the best policy.

The problem of balancing exploration and exploitation in RL has been successfully addressed for MDPs and algorithms with near optimal regret bounds known \citep{bartlett2009regal,jaksch2010near,ouyang2017learning,azar2017minimax,fruit2018efficient,jin2018q,abbasi2019exploration,zhang2019regret,zanette2019tighter,hao2020provably,wei2020model,wei2021learning}.
MDPs assume that the state is perfectly observable by the agent and the only uncertainty is about the underlying dynamics of the environment. However, in many real-world scenarios such as robotics, healthcare and finance, the state is not fully observed by the agent, and only a partial observation is available. These scenarios are modeled by Partially Observable Markov Decision Processes (POMDPs). In addition to the uncertainty in the environment dynamics, the agent has to deal with the uncertainty about the underlying state. It is well known \citep{kumar2015stochastic} that introducing an information or belief state (a posterior distribution over the states given the history of observations and actions) allows the POMDP to be recast as an MDP over the belief state space. The resulting algorithm requires a posterior update of the belief state which needs the transition and observation model to be fully known. This presents a significant difficulty when the model parameters are unknown. Thus, managing the exploration-exploitation trade-off for POMDPs is a significant challenge and to the best of our knowledge, no online RL algorithm with sub-linear regret is known.

In this paper, we consider infinite-horizon average-cost POMDPs with finite states, actions and observations. The underlying state transition dynamics is unknown, though we assume the observation kernel to be known. We propose a Posterior Sampling Reinforcement Learning algorithm (\texttt{PSRL-POMDP}) and prove that it achieves a Bayesian expected regret bound of $\order(\log T)$ in the finite (transition kernel) parameter set case where $T$ is the time horizon. We then show that in the general (continuous parameter set) case, it achieves $\otil(T^{2/3})$ under some technical assumptions.
The \texttt{PSRL-POMDP} algorithm is a natural extension of the \texttt{TSDE} algorithm for MDPs \citep{ouyang2017learning} with two main differences. First, in addition to the posterior distribution on the environment dynamics, the algorithm maintains a posterior distribution on the underlying state. Second, since the state is not fully observable, the agent cannot keep track of the number of visits to state-action pairs, a quantity that is crucial in the design of algorithms for tabular MDPs. Instead, we introduce a notion of \textit{pseudo count} and carefully handle its relation with the true counts to obtain sub-linear regret. To the best of our knowledge, \texttt{PSRL-POMDP} is the first online RL algorithm for POMDPs with sub-linear regret.

\subsection{Related Literature}
\label{sec: related literature}

We review the related literature in two main domains: efficient exploration for MDPs, and learning in POMDPs. 

\noindent\textbf{Efficient exploration in MDPs.} To balance the exploration and exploitation, two general techniques are used in the basic tabular MDPs: \textit{optimism in the face of uncertainty (OFU)}, and \textit{posterior sampling}. Under the OFU technique, the agent constructs a confidence set around the system parameters, selects an optimistic parameter associated with the minimum cost from the confidence set, and takes actions with respect to the optimistic parameter. This principle is widely used in the literature to achieve optimal regret bounds
\citep{bartlett2009regal,jaksch2010near,azar2017minimax,fruit2018efficient,jin2018q,zhang2019regret,zanette2019tighter,wei2020model}.
An alternative technique to encourage exploration is posterior sampling \citep{thompson1933likelihood}. In this approach, the agent maintains a posterior distribution over the system parameters, samples a parameter from the posterior distribution, and takes action with respect to the sampled parameter \citep{strens2000bayesian,osband2013more,fonteneau2013optimistic,gopalan2015thompson,ouyang2017learning,jafarnia2021online}. In particular, \citep{ouyang2017learning} proposes  \texttt{TSDE}, a posterior sampling-based  algorithm for the infinite-horizon average-cost MDPs.

Extending these results to the continuous state MDPs has been recently addressed with general function approximation \citep{osband2014model,dong2020root,ayoub2020model,wang2020reinforcement}, or in the special cases of linear function approximation \citep{abbasi2019politex,abbasi2019exploration,jin2020provably,hao2020provably,wei2021learning,wang2021provably}, and Linear Quadratic Regulators \citep{ouyang2017learningbased,dean2018regret,cohen2019learning,mania2019certainty,simchowitz2020naive,lale2020explore}. In general, POMDPs can be formulated as continuous state MDPs by considering the belief as the state. However, computing the belief requires the knowledge of the model parameters and thus unobserved in the RL setting. Hence, learning algorithms for continuous state MDPs cannot be directly applied to POMDPs. 

\noindent\textbf{Learning in POMDPs.} To the best of our knowledge, the only existing work with sub-linear regret in POMDPs is \cite{azizzadenesheli2017experimental}. However, their definition of regret is not with respect to the optimal policy, but with respect to the best memoryless policy (a policy that maps the current observation to an action). With our natural definition of regret, their algorithm suffers linear regret. Other learning algorithms for POMDPs either consider linear dynamics \citep{lale2020logarithmic,tsiamis2020online} or do not consider regret \citep{shani2005model,ross2007bayes,poupart2008model,cai2009learning,liu2011infinite,liu2013online,doshi2013bayesian,katt2018bayesian,azizzadenesheli2018policy} and are not directly comparable to our setting.

\section{Preliminaries}
\label{sec: preliminaries}

An infinite-horizon average-cost Partially Observable Markov Decision Process (POMDP) can be specified by $(\calS, \calA, \theta, C, O, \eta)$ where $\calS$ is the state space, $\calA$ is the action space, $C: \calS \times \calA \to [0, 1]$ is the cost function, and $O$ is the set of observations. Here $\eta: \calS \to \Delta_O$ is the observation kernel, and $\theta: \calS \times \calA \to \Delta_\calS$ is the transition kernel such that $\eta(o | s) = \p(o_t=o | s_t=s)$ and $\theta(s'|s, a) = \p(s_{t+1}=s'|s_t=s, a_t=a)$ where $o_t \in O$, $s_t \in \calS$ and $a_t \in \calA$ are the observation, state and action at time $t=1, 2, 3, \cdots$. Here, for a finite set $\calX$,  $\Delta_\calX$ is the set of all probability distributions on $\calX$. We assume that the state space, the action space and the observations are finite with size $|\calS|, |\calA|, |O|$, respectively. 

Let $\calF_t$ be the information available at time $t$ (prior to action $a_t$), i.e., the sigma algebra generated by the history of actions and observations $a_1, o_1, \cdots, a_{t-1}, o_{t-1}, o_t$ and let $\calF_{t+}$ be the information after choosing action $a_t$. Unlike MDPs, the state is not observable by the agent and the optimal policy cannot be a function of the state. Instead, the agent maintains a belief $h_t(\cdot;\theta) \in \Delta_\calS$ given by $h_t(s ; \theta) := \p(s_t=s | \calF_t;\theta)$, as a sufficient statistic for the history of observations and actions. Here we use the notation $h_t(\cdot ; \theta)$ to explicitly show the dependency of the belief on $\theta$. After taking action $a_t$ and observing $o_{t+1}$, the belief $h_t$ can be updated as
\begin{equation}
\label{eq: update rule of ht}
h_{t+1}(s';\theta) =\frac{\sum_{s}\eta(o_{t+1}|s') \theta(s'|s, a_t)h_t(s;\theta)}{\sum_{s'}\sum_{s}\eta(o_{t+1}|s') \theta(s'|s, a_t)h_t(s;\theta)}.
\end{equation}
This update rule is compactly denoted by $h_{t+1}(\cdot ; \theta) = \tau (h_t(\cdot;\theta), a_t, o_{t+1}; \theta)$, with the initial condition
\begin{equation*}
h_1(s;\theta) = \frac{\eta(o_1|s)h(s;\theta)}{\sum_s\eta(o_1|s)h(s;\theta)},
\end{equation*}
 where $h(\cdot;\theta)$ is the distribution of the initial state $s_1$ (denoted by $s_1 \sim h$). A deterministic stationary policy $\pi : \Delta_\calS \to \calA$ maps a belief to an action. The long-term average cost of a policy $\pi$ can be defined as
\begin{equation}
J_\pi(h; \theta) := \limsup_{T \to \infty} \frac{1}{T} \sum_{t=1}^T \E\Big[C\Big(s_t, \pi\big(h_t(\cdot;\theta)\big)\Big) \Big].
\end{equation}
Let  $J(h, \theta) := \inf_\pi J_\pi(h, \theta)$ be the optimal long-term average cost that in general may depend on the initial state distribution $h$, though we will assume it is independent of the initial distribution $h$ (and thus denoted by $J(\theta)$), and the following Bellman equation holds:
\begin{assumption}[Bellman optimality equation]
\label{ass: bellman equation}
There exist $J(\theta) \in \fR$ and a bounded function $v(\cdot; \theta):\Delta_\calS \to \fR$ such that for all $b \in \Delta_\calS$
\begin{equation}
\label{eq: bellman equation}
J(\theta) + v(b; \theta) = \min_{a \in \calA} \{c(b, a) + \sum_{o \in O}P(o | b, a; \theta)v(b'; \theta)\},
\end{equation}
\end{assumption}
where $v$ is called the relative value function, $b' = \tau(b, a, o; \theta)$ is the updated belief, $c(b, a) := \sum_s C(s, a)b(s)$ is the expected cost, and $P(o|b, a; \theta)$ is the probability of observing $o$ in the next step, conditioned on the current belief $b$ and action $a$, i.e.,
\begin{equation}
P(o | b, a; \theta) = \sum_{s' \in \calS}\sum_{s \in \calS} \eta(o | s') \theta(s' | s, a) b(s).
\end{equation}
Various conditions are known under which Assumption \ref{ass: bellman equation} holds, e.g.,  when the MDP is weakly communicating \citep{bertsekas2017dynamic}.
Note that if Assumption~\ref{ass: bellman equation} holds, the policy $\pi^*$ that minimizes the right hand side of \eqref{eq: bellman equation} is the optimal policy. More precisely,
\begin{lemma}
\label{lem: optimal policy}
Suppose Assumption~\ref{ass: bellman equation} holds. Then, the policy $\pi^*(\cdot, \theta):\Delta_\calS \to \calA$ given by
\begin{equation}
\label{eq: the optimal policy}
\pi^*(b; \theta) := \argmin_{a \in \calA} \{c(b, a) + \sum_{o \in O}P(o | b, a; \theta)v(b'; \theta)\}
\end{equation}
is the optimal policy with $J_{\pi^*}(h; \theta) = J(\theta)$ for all $h \in \Delta_\calS$.
\end{lemma}


Note that if $v$ satisfies the Bellman equation, so does $v$ plus any constant. Therefore, without loss of generality, and since $v$ is bounded, we can assume that $\inf_{b \in \Delta_\calS}v(b; \theta) = 0$ and define the span of a POMDP as $\spn(\theta) := \sup_{b \in \Delta_\calS}v(b; \theta)$. Let $\Theta_H$ be the class of POMDPs that satisfy Assumption~\ref{ass: bellman equation} and have $\spn(\theta) \leq H$ for all $\theta \in \Theta_H$. In Section~\ref{sec: finite}, we consider a finite subset $\Theta \subseteq \Theta_H$ of POMDPs.  In Section~\ref{sec: general}, the general class $\Theta = \Theta_H$ is considered. 

\paragraph{The learning protocol.} We consider the problem of an agent interacting with an unknown randomly generated POMDP $\theta_*$, where $\theta_* \in \Theta$ is randomly generated according to the probability distribution $f(\cdot)$.\footnote{In Section~\ref{sec: finite}, $f(\cdot)$ should be viewed as a probability mass function.} After the initial generation of $\theta_*$, it remains fixed, but unknown to the agent. The agent interacts with the POMDP $\theta_*$ in $T$ steps. Initially, the agent starts from state $s_1$ that is randomly generated according to the conditional probability mass function $h(\cdot;\theta_*)$. At time $t=1, 2, 3, \cdots, T$, the agent observes $o_t \sim \eta(\cdot|s_t)$, takes action $a_t$ and suffers cost of $C(s_t, a_t)$. The environment, then determines the next state $s_{t+1}$ which is randomly drawn from the probability distribution $\theta_*(\cdot | s_t, a_t)$. Note that although the cost function $C$ is assumed to be known, the agent cannot observe the value of $C(s_t, a_t)$ since the state $s_t$ is unknown to the agent. The goal of the agent is to minimize the expected cumulative regret defined as
\begin{equation}
R_T := \Ec\Big[\sum_{t=1}^T\Big[C(s_t, a_t) - J(\theta_*)\Big]\Big],
\end{equation}
where the expectation is with respect to the prior distribution $h(\cdot;\theta_*)$ for $s_1$, the randomness in the state transitions, and the randomness in the algorithm. Here, $\Ec[\cdot]$ is a shorthand for $\E[\cdot|\theta_*]$. In Section~\ref{sec: finite}, a regret bound is provided on $R_T$, however, Section~\ref{sec: general} considers $\E[R_T]$ (also called Bayesian regret) as the performance measure for the learning algorithm. We note that the Bayesian regret is widely considered in the MDP literature \citep{osband2013more,gopalan2015thompson,ouyang2017learning,ouyang2017learningbased}.


\section{The \texttt{PSRL-POMDP} Algorithm}
\label{sec: posterior sampling algorithm}

We propose a general Posterior Sampling Reinforcement Learning for POMDPs (\texttt{PSRL-POMDP}) algorithm (Algorithm~\ref{alg: posterior sampling}) for both the finite-parameter and the general case. The algorithm maintains a joint distribution on the unknown parameter $\theta_*$ as well as the state $s_t$. \texttt{PSRL-POMDP} takes the prior distributions $h$ and $f$ as input. At time $t$, the agent computes the posterior distribution $f_t(\cdot)$ on the unknown parameter $\theta_*$ as well as the posterior conditional probability mass function (pmf) $h_t(\cdot;\theta)$ on the state $s_t$ for $\theta \in \Theta$. Upon taking action $a_t$ and observing $o_{t+1}$, the posterior distribution at time $t+1$ can be updated by applying the Bayes' rule as\footnote{When the parameter set is finite, $\int_{\theta}$ should be replaced with $\sum_\theta$.}
\begin{align}
\label{eq: update rule of ft and ht}
f_{t+1}(\theta) &= \frac{\sum_{s, s'}\eta(o_{t+1}|s') \theta(s'|s, a_t)h_t(s;\theta)f_t(\theta)}{\int_{\theta} \sum_{s,s'}\eta(o_{t+1}|s') \theta(s'|s, a_t)h_t(s;\theta)f_t(\theta)d\theta}, \nonumber \\
h_{t+1}(\cdot ;\theta) &= \tau (h_t(\cdot;\theta), a_t, o_{t+1};\theta),
\end{align}
with the initial condition
\begin{align}
\label{eq: initial update of ht and ft}
f_1(\theta) &= \frac{\sum_s \eta(o_1|s)h(s;\theta)f(\theta)}{\int_{\theta} \sum_s \eta(o_1|s)h(s;\theta)f(\theta)d\theta}, \nonumber \\
h_1(s;\theta) &= \frac{\eta(o_1|s)h(s;\theta)}{\sum_s\eta(o_1|s)h(s;\theta)}.
\end{align}
Recall that $\tau (h_t(\cdot;\theta), a_t, o_{t+1};\theta)$ is a compact notation for \eqref{eq: update rule of ht}. 
In the special case of perfect observation at time $t$, $h_{t}(s;\theta) = \1(s_{t}=s)$ for all $\theta \in \Theta$ and $s \in \calS$. Moreover, the update rule of $f_{t+1}$ reduces to that of fully observable MDPs (see Eq. (4) of \cite{ouyang2017learning}) in the special case of perfect observation at time $t$ and $t+1$.

Let $n_{t}(s, a) = \sum_{\tau=1}^{t-1} \1(s_\tau=s, a_\tau=a)$ be the number of visits to state-action $(s, a)$ by time $t$. The number of visits $n_{t}$ plays an important role in  learning for MDPs \citep{jaksch2010near,ouyang2017learning} and is one of the two criteria to determine the length of the episodes in the \texttt{TSDE} algorithm for MDPs \citep{ouyang2017learning}. However, in POMDPs, $n_{t}$ is not $\calF_{{(t-1)}+}$-measurable since the states are not observable. Instead, let $\tilde n_{t}(s, a) := \E[n_{t}(s, a) | \calF_{{(t-1)}+}]$, and define the \textit{pseudo-count} $\tilde m_t$ as follows:
\begin{definition}
$(\tilde m_t)_{t=1}^T$ is a pseudo-count if it is a non-decreasing, integer-valued sequence of random variables such that $\tilde m_t$ is $\calF_{{(t-1)}+}$-measurable, $\tilde m_t(s, a) \geq \ceil{\tilde n_t(s, a)}$, and $\tilde m_t(s, a) \leq t$ for all $t \leq T+1$.
\end{definition}
An example of such a sequence is simply $\tilde m_t(s, a) = t$ for all state-action pair $(s, a) \in \calS \times \calA$. This is used in Section~\ref{sec: finite}. Another example is $\tilde m_t(s, a) := \max \{\tilde m_{t-1}(s, a), \ceil{\tilde n_t(s, a)}\}$ with $\tilde m_0(s, a) = 0$ for all $(s, a) \in \calS \times \calA$ which is used in Section~\ref{sec: general}. Here $\ceil{\tilde n_t(s, a)}$ is the smallest integer that is greater than or equal to $\tilde n_t(s, a)$. By definition, $\tilde m_t$ is integer-valued and non-decreasing which is essential to bound the number of episodes in the algorithm for the general case (see Lemma~\ref{lem: number of episodes}).

\begin{algorithm}
\caption{\textsc{\texttt{PSRL-POMDP}}}
\label{alg: posterior sampling}
\begin{algorithmic}[1]
\REQUIRE prior distributions $f(\cdot), h(\cdot;\cdot)$ \\
\textbf{Initialization: }$t \gets 1, t_1 \gets 0$\\
Observe $o_1$ and compute $f_1, h_1$ according to \eqref{eq: initial update of ht and ft} \\
\FOR{{\normalfont episodes} $k=1, 2, \cdots$}
\STATE	$T_{k-1} \gets t - t_k$\\
\STATE	$t_k \gets t$\\
\STATE	Generate $\theta_k \sim f_{t_k}(\cdot)$ and compute $\pi_k(\cdot) = \pi^*(\cdot;\theta_k)$ from \eqref{eq: the optimal policy}\\
	\WHILE{$t \leq \sched(t_k, T_{k-1})$ and $\tilde m_{t}(s, a) \leq 2 \tilde m_{t_k}(s, a)$ for all $(s, a) \in \calS \times \calA$}
	\STATE Choose action $a_t = \pi_k(h_t(\cdot;\theta_k))$ and observe $o_{t+1}$\\
	\STATE Update $f_{t+1}, h_{t+1}$ according to \eqref{eq: update rule of ft and ht}\\
	\STATE $t \gets t+1$	
	\ENDWHILE
\ENDFOR
\end{algorithmic}
\end{algorithm}


Similar to the \texttt{TSDE} algorithm for fully observable MDPs, \texttt{PSRL-POMDP} algorithm proceeds in episodes. In the beginning of episode $k$, POMDP $\theta_k$ is sampled from the posterior distribution $f_{t_k}$ where $t_k$ denotes the start time of episode $k$. The optimal policy $\pi^*(\cdot; \theta_k)$ is then computed and used during the episode. Note that the input of the policy is $h_t(\cdot;\theta_k)$. The intuition behind such a choice (as opposed to the belief $b_t(\cdot) := \int_{\theta} h_t(\cdot;\theta)f_t(\theta)d\theta$) is that during episode $k$, the agent treats $\theta_k$ to be the true POMDP and adopts the optimal policy with respect to it. Consequently, the input to the policy should also be the conditional belief with respect to the sampled $\theta_k$. 

A key factor in designing posterior sampling based algorithms is the design of episodes. Let $T_k$ denote the length of episode $k$. In \texttt{PSRL-POMDP}, a new episode starts if either $t > \sched(t_k, T_{k-1})$ or $\tilde m_{t}(s, a) > 2\tilde m_{t_k}(s, a)$. In the finite parameter case (Section~\ref{sec: finite}), we consider $\sched(t_k, T_{k-1}) = 2t_k$ and $\tilde m_t(s, a) = t$. With these choices, the two criteria coincide and ensure that the start time and the length of the episodes are deterministic. In Section~\ref{sec: general}, we use $\sched(t_k, T_{k-1}) = t_k + T_{k-1}$ and $\tilde m_t(s, a) := \max \{\tilde m_{t-1}(s, a), \ceil{\tilde n_t(s, a)}\}$. This guarantees that $T_k \leq T_{k-1} + 1$ and $\tilde m_t(s, a) \leq 2\tilde m_{t_k}(s, a)$. These criteria are previously introduced in the \texttt{TSDE} algorithm \citep{ouyang2017learning} except that \texttt{TSDE} uses the true count $n_t$ rather than $\tilde m_t$.


\section{Finite-Parameter Case ($|\Theta| < \infty$)}
\label{sec: finite}

In this section, we consider $\Theta \subseteq \Theta_H$ such that $|\Theta| < \infty$. When $\Theta$ is finite, the posterior distribution concentrates on the true parameter exponentially fast if the transition kernels are separated enough (see Lemma~\ref{lem: 1-ft}). This allows us to achieve a regret bound of $\order(H\log T)$. Let $o_{1:t}, a_{1:t}$ be shorthand for the history of observations $o_1, \cdots, o_t$ and the history of actions $a_1, \cdots, a_t$, respectively. Let $\vtheta(o)$ be the probability of observing $o$ at time $t+1$ if the action history is $a_{1:t}$, the observation history is $o_{1:t}$, and the underlying transition kernel is $\theta$, i.e.,
\begin{align*}
\vtheta(o) := \p(o_{t+1}=o | o_{1:t}, a_{1:t}, \theta_*=\theta).
\end{align*}
The distance between $\vtheta$ and $\vgamma$ is defined by Kullback Leibler (KL-) divergence as follows. For a fixed state-action pair $(s, a)$ and any $\theta, \gamma \in \Theta$, denote by $\calK(\vtheta \| \vgamma)$, the Kullback Leibler (KL-) divergence between the probability distributions $\vtheta$ and $\vgamma$ is given by
\begin{align*}
\calK(\vtheta \| \vgamma) := \sum_{o} \vtheta(o) \log\frac{\vtheta(o)}{\vgamma(o)}.
\end{align*}
It can be shown that $\calK(\vtheta \| \vgamma) \geq 0$ and that equality holds if and only if $\vtheta = \vgamma$. Thus, KL-divergence can be thought of as a measure of divergence of $\vgamma$ from $\vtheta$. In this section, we need to assume that the transition kernels in $\Theta$ are distant enough in the following sense.
\begin{assumption}
\label{ass: kl divergence}
For any time step $t$, any history of observations $o_{1:t}$ and actions $a_{1:t}$, and any two transition kernels $\theta, \gamma \in \Theta$, there exists a positive constant $\epsilon > 0$ such that $\calK(\vtheta \| \vgamma) \geq \epsilon$.
\end{assumption}
This assumption is similar to the one used in \cite{kim2017thompson}.

\begin{theorem}
\label{thm: finite parameter regret}
Suppose Assumptions~\ref{ass: bellman equation} and \ref{ass: kl divergence} hold. Then, the regret bound of Algorithm~\ref{alg: posterior sampling} with $\sched(t_k, T_{k-1}) = 2t_k$ and $\tilde m_t(s, a) = t$ for all state-action pairs $(s, a)$ is bounded as
\begin{align*}
R_T \leq H\log T + \frac{4(H+1)}{(e^{-\beta}-1)^2},
\end{align*}
where $\beta > 0$ is a universal constant defined in Lemma~\ref{lem: 1-ft}.
\end{theorem}
Observe that with $\sched(t_k, T_{k-1}) = 2t_k$ and $\tilde m_t(s, a) = t$, the two stopping criteria in Algorithm~\ref{alg: posterior sampling} coincide and ensure that $T_k = 2T_{k-1}$ with $T_0 = 1$. In other words, the length of episodes grows exponentially as $T_k = 2^k$.

\subsection{Proof of Theorem~\ref{thm: finite parameter regret}}
In this section, proof of Theorem~\ref{thm: finite parameter regret} is provided. A key factor in achieving $\order(H\log T)$ regret bound in the case of finite parameters is that the posterior distribution $f_t(\cdot)$ concentrates on the true $\theta_*$ exponentially fast. 
\begin{lemma}
\label{lem: 1-ft}
Suppose Assumption~\ref{ass: kl divergence} holds. Then, there exist constants $\alpha > 1$ and $\beta > 0$ such that
\begin{align*}
\E[1 - f_t(\theta_*) | \theta_*] \leq \alpha \exp(-\beta t).
\end{align*}
\end{lemma}

Equipped with this lemma, we are now ready to prove Theorem \ref{thm: finite parameter regret}.
\begin{proof}
Note that the regret $R_T$ can be decomposed as $R_T = H\Ec[K_T] + R_1 + R_2 + R_3$, where
\begin{align*}
R_1 &:= \Ec\sbr{\sum_{k=1}^{K_T}T_k\Big[J(\theta_k) - J(\theta_*)\Big]}, \\
R_2 &:= H\Ec\sbr{\sum_{k=1}^{K_T}\sum_{t=t_k}^{t_{k+1}-1}\sbr{\sum_{s'}\abr{\theta_*(s'|s_t, a_t) - \theta_k(s'|s_t, a_t)} + \sum_s \abr{h_t(s;\theta_*) - h_t(s;\theta_k)}}}, \\
R_3 &:= \Ec\sbr{\sum_{k=1}^{K_T}\sum_{t=t_k}^{t_{k+1}-1}\Big[c(h_t(\cdot;\theta_*), a_t) - c(h_t(\cdot;\theta_k), a_t)\Big]}.
\end{align*}
Note that the start time and length of episodes in Algorithm~\ref{alg: posterior sampling} are deterministic with the choice of $\sched$ and $\tilde m_t$ in the statement of the theorem, i.e., $t_k$, $T_k$ and hence $K_T$ are deterministic. Note that if $\theta_k = \theta_*$, then $R_1 = R_2 = R_3 = 0$. Moreover, we have that $J(\theta_k) - J(\theta_*) \leq 1$, $\sum_{s'}\abr{\theta_*(s'|s_t, a_t) - \theta_k(s'|s_t, a_t)} \leq 2$, $\sum_s \abr{h_t(s;\theta_*) - h_t(s;\theta_k)} \leq 2$, and $c(h_t(\cdot;\theta_*), a_t) - c(h_t(\cdot;\theta_k), a_t) \leq 1$. Therefore,
\begin{align*}
R_1 &:= \Ec\sbr{\sum_{k=1}^{K_T}T_k\1(\theta_k \neq \theta_*)} = \sum_{k=1}^{K_T}T_k \pc(\theta_k \neq \theta_*), \\
R_2 &:= 4H\Ec\sbr{\sum_{k=1}^{K_T}\sum_{t=t_k}^{t_{k+1}-1}\1(\theta_k \neq \theta_*)} = 4H\sum_{k=1}^{K_T}T_k \pc(\theta_k \neq \theta_*), \\
R_3 &:= \Ec\sbr{\sum_{k=1}^{K_T}\sum_{t=t_k}^{t_{k+1}-1}\1(\theta_k \neq \theta_*)} = \sum_{k=1}^{K_T}T_k \pc(\theta_k \neq \theta_*).
\end{align*}
Note that $\pc(\theta_k \neq \theta_*) = \Ec[1 - f_{t_k}(\theta_*)] \leq \alpha \exp(-\beta t_k)$ by Lemma~\ref{lem: 1-ft}. Combining all these bounds, we can write
\begin{align*}
R_T \leq HK_T + (4H+2)\alpha \sum_{k=1}^{K_T}T_k \exp(-\beta t_k).
\end{align*}
With the episode schedule provided in the statement of the theorem, it is easy to check that $K_T = O(\log T)$. 
Let $n = 2^{K_T}$ and write
\begin{align*}
\sum_{k=1}^{K_T}T_k \exp(-\beta t_k) = \sum_{k=1}^{K_T}2^k e^{-\beta (2^k-1)} \leq \sum_{j=2}^{n}je^{-\beta (j-1)} = \frac{d}{dx} \frac{x^{n+1}-1}{x-1} \Big|_{x = e^{-\beta}}-1.
\end{align*}
The last equality is by geometric series. Simplifying the derivative yields
\begin{align*}
\frac{d}{dx} \frac{x^{n+1}-1}{x-1} \Big|_{x = e^{-\beta}} &= \frac{nx^{n+1} - (n+1)x^n + 1}{(x-1)^2}\Big|_{x = e^{-\beta}} \leq \frac{nx^{n} - (n+1)x^n + 1}{(x-1)^2}\Big|_{x = e^{-\beta}} \\
&= \frac{-x^n + 1}{(x-1)^2}\Big|_{x = e^{-\beta}} \leq \frac{2}{(e^{-\beta}-1)^2}.
\end{align*}
Substituting these values implies $R_T \leq H\log T + \frac{4(H+1)}{(e^{-\beta}-1)^2}$.
\end{proof}

\section{General Case ($|\Theta| = \infty$)}
\label{sec: general}

We now consider the general case, where the parameter set is infinite, and in particular, $\Theta = \Theta_H$, an uncountable set. 
We make the following two technical assumptions on the belief and the transition kernel.
\begin{assumption}
\label{ass: concentrating belief}
Denote by $k(t)$ the episode at time $t$. The true conditional belief $h_t(\cdot;\theta_*)$ and the approximate conditional belief $h_t(\cdot;\theta_{k(t)})$ satisfy
\begin{equation}
\label{eq: ass concentrating belief}
\E\Big[\sum_{s}\big|h_t(s;\theta_*) - h_t(s;\theta_{k(t)})\big|\Big] \leq \frac{K_1(|\calS|, |\calA|, |O|, \iota)}{\sqrt{t_{k(t)}}},
\end{equation}
with probability at least $1 - \delta$, for any $\delta \in (0, 1)$. Here $K_1(|\calS|, |\calA|, |O|, \iota)$ is a constant that is polynomial in its input parameters and $\iota$ hides the logarithmic dependency on $|\calS|, |\calA|, |O|, T, \delta$.
\end{assumption}
Assumption~\ref{ass: concentrating belief} states that the gap between conditional posterior function for the sampled POMDP $\theta_k$ and the true POMDP $\theta_*$ decreases with episodes as better approximation of the true POMDP is available. 
There has been recent work on computation of approximate information states as required in Assumption \ref{ass: concentrating belief} \citep{subramanian2020approximate}.

\begin{assumption}
\label{ass: concentration}
There exists an $\calF_{t}$-measurable estimator $\hat \theta_t : \calS \times \calA \to \Delta_\calS$ such that
\begin{align}
\label{eq: ass concentration}
\sum_{s'}|\theta_*(s'|s, a) - \hat \theta_t(s' | s, a)| \leq \frac{K_2(|\calS|, |\calA|, |O|, \iota)}{\sqrt{\max\{1, {\tilde m}_{t}(s, a)\}}}
\end{align}
with probability at least $1 - \delta$, for any $\delta \in (0, 1)$, uniformly for all $t=1, 2, 3, \cdots, T$, where $K_2(|\calS|, |\calA|, |O|, \iota)$ is a constant that is polynomial in its input parameters and $\iota$ hides the logarithmic dependency on $|\calS|, |\calA|, |O|, T, \delta$.
\end{assumption}
There has been extensive work on estimation of transition dynamics of MDPs, e.g., \citep{grunewalder2012modelling}. Two examples where Assumptions~\ref{ass: concentrating belief} and \ref{ass: concentration} hold are:
\begin{itemize}
\item \textbf{Perfect observation.} In the case of perfect observation, where $h_t(s;\theta) = \1(s_t=s)$, Assumption~\ref{ass: concentrating belief} is clearly satisfied. Moreover, with perfect observation, one can choose $\tilde{m}_{t}(s, a) = n_{t}(s, a)$ and select  $\hat \theta_k(s'|s, a) = \frac{n_{t}(s, a, s')}{n_{t}(s, a)}$ to satisfy Assumption~\ref{ass: concentration} \citep{jaksch2010near,ouyang2017learning}. Here $n_{t}(s, a, s')$ denotes the number of visits to $s, a$ such that the next state is $s'$ before time $t$.
\item \textbf{Finite-parameter case.} In the finite-parameter case with the choice of $\tilde m_t(s, a) = t$ for all state-action pairs $(s, a)$ and $\sched(t_k, T_{k-1}) = t_k + T_{k-1}$ or $\sched(t_k, T_{k=1}) = 2t_k$, both of the assumptions are satisfied (see Lemma~\ref{lem: assumptions satisfied for finite-parameter} for details). Note that in this case a more refined analysis is performed in Section~\ref{sec: finite} to achieve $\order(H\log T)$ regret bound.
\end{itemize}
Now, we state the main result of this section.
\begin{theorem}
\label{thm: regret bound}
Under Assumptions~\ref{ass: bellman equation}, \ref{ass: concentrating belief} and \ref{ass: concentration}, running \texttt{PSRL-POMDP} algorithm with $\sched(t_k, T_{k-1}) = t_k + T_{k-1}$ yields $\E[R_T] \leq \otil(HK_2(|\calS||\calA|T)^{2/3})$, where $K_2 := K_2(|\calS|, |\calA|, |O|, \iota)$ in Assumption~\ref{ass: concentration}.
\end{theorem}
The exact constants are known (see proof and Appendix~\ref{app: full upper bound}) though we have hidden the dependence above.

\subsection{Proof Sketch of Theorem~\ref{thm: regret bound}}
We provide the proof sketch of Theorem~\ref{thm: regret bound} here. 
A key property of posterior sampling is that conditioned on the information at time $t$, the sampled $\theta_t$ and the true $\theta_*$ have the same distribution \citep{osband2013more,russo2014learning}. Since the episode start time $t_k$ is a stopping time with respect to the filtration $(\calF_t)_{t \geq 1}$, we use a stopping time version of this property:
\begin{lemma}[Lemma 2 in \cite{ouyang2017learning}]
\label{lem: property of posterior sampling}
For any measurable function $g$ and any $\calF_{t_k}$-measurable random variable $X$, we have $\E[g(\theta_k, X)] = \E[g(\theta_*, X)]$.
\end{lemma}
Introducing the pseudo count $\tilde m_t(s, a)$ in the algorithm requires a novel analysis to achieve a low regret bound. The following key lemma states that the pseudo count $\tilde m_{t}$ cannot be too smaller than the true count  $n_t$.
\begin{lemma}
\label{lem: pseudo count vs true count general case}
Fix a state-action pair $(s, a) \in \calS \times \calA$. For any pseudo count $\tilde m_t$ and any $\alpha \in [0, 1]$,
\begin{equation}
\p\big(\tilde m_{t}(s, a) < \alpha n_{t}(s, a)\big) \leq \alpha.
\end{equation}
\end{lemma}
\begin{proof}
We show that $\p\big(\tilde n_{t}(s, a) < \alpha n_{t}(s, a)\big) \leq \alpha$. Since by definition $\tilde m_t(s, a) \geq \tilde n_t(s, a)$, the claim of the lemma follows. For any $\alpha \in [0, 1]$,
\begin{equation}
\tilde n_{t}(s, a) \1\big(\alpha n_{t}(s, a) > \tilde n_{t}(s, a)\big) \leq \alpha n_{t}(s, a).
\end{equation}
By taking conditional expectation with respect to $\calF_{{(t-1)}+}$ from both sides and the fact that $\E[n_{t}(s, a)|\calF_{{(t-1)}+}] = \tilde n_{t}(s, a)$, we have
\begin{equation}
\label{eq: ntilde indicator smaller than ntilde}
\tilde n_{t}(s, a) \E\Big[\1\big(\alpha n_{t}(s, a) > \tilde n_{t}(s, a)\big) \Big| \calF_{{(t-1)}+} \Big] \leq \alpha \tilde n_{t}(s, a).
\end{equation}
We claim that
\begin{equation}
\label{eq: indicator less than alpha}
\E\Big[\1\big(\alpha n_{t}(s, a) > \tilde n_{t}(s, a)\big) \Big| \calF_{{(t-1)}+}\Big] \leq \alpha,~~\text{a.s.}
\end{equation}
If this claim is true, taking another expectation from both sides completes the proof.

To prove the claim, let $\Omega_0, \Omega_+$ be the subsets of the sample space where $\tilde n_t(s, a) = 0$ and $\tilde n_t(s, a) > 0$, respectively. We consider these two cases separately: (a) on $\Omega_+$ one can divide both sides of \eqref{eq: ntilde indicator smaller than ntilde} by  $\tilde n_t(s, a)$ and reach \eqref{eq: indicator less than alpha}; (b) note that by definition $\tilde n_t(s, a) = 0$ on $\Omega_0$. Thus, $ n_t(s, a)\1(\Omega_0) = 0$ almost surely (this is because $\E[n_t(s, a)\1(\Omega_0)] = \E[\E[n_t(s, a)\1(\Omega_0)|\calF_{(t-1)+}]] = \E[\tilde n_t(s, a)\1(\Omega_0)] = 0$). Therefore,
\begin{align*}
\1(\Omega_0)\1\big(\alpha n_{t}(s, a) > \tilde n_{t}(s, a)\big) = 0,\quad\text{a.s.,}
\end{align*}
which implies
\begin{equation*}
\1(\Omega_0)\E\Big[\1\big(\alpha n_{t}(s, a) > \tilde n_{t}(s, a)\big) \Big| \calF_{{(t-1)}+}\Big] = 0,\quad\text{a.s.,}
\end{equation*}
which means on $\Omega_0$, the left hand side of \eqref{eq: indicator less than alpha} is indeed zero, almost surely, proving  the claim.
\end{proof}
The parameter $\alpha$ will be tuned later to balance two terms and achieve $\otil(T^{2/3})$ regret bound (see Lemma~\ref{lem: r2bar}). We are now ready to provide the proof sketch of Theorem~\ref{thm: regret bound}.

By Lemma~\ref{lem: decomposition}, $R_T$ can be decomposed as $R_T = H\Ec[K_T] + R_1 + R_2 + R_3$, where
\begin{align*}
R_1 &:= \Ec\sbr{\sum_{k=1}^{K_T}T_k\Big[J(\theta_k) - J(\theta_*)\Big]}, \\
R_2 &:= H\Ec\sbr{\sum_{k=1}^{K_T}\sum_{t=t_k}^{t_{k+1}-1}\sbr{\sum_{s'}\abr{\theta_*(s'|s_t, a_t) - \theta_k(s'|s_t, a_t)} + \sum_s \abr{h_t(s;\theta_*) - h_t(s;\theta_k)}}}, \\
R_3 &:= \Ec\sbr{\sum_{k=1}^{K_T}\sum_{t=t_k}^{t_{k+1}-1}\Big[c(h_t(\cdot;\theta_*), a_t) - c(h_t(\cdot;\theta_k), a_t)\Big]}.
\end{align*}
It follows from the first stopping criterion that $T_k \leq T_{k-1} + 1$. Using this along with the property of posterior sampling (Lemma~\ref{lem: property of posterior sampling}) proves that $\E[R_1] \leq \E[K_T]$ (see Lemma~\ref{lem: j minus jstar} for details). $\E[R_3]$ is bounded by $K_1\E\big[\sum_{k=1}^{K_T}\frac{T_k}{\sqrt{t_k}}\big] + 1$ where $K_1:= K_1(|\calS|, |\calA|, |O|, \iota)$ is the constant in Assumption~\ref{ass: concentrating belief} (see Lemma~\ref{lem: r3}). To bound $\E[R_2]$, we use Assumption~\ref{ass: concentrating belief} and follow the proof steps of Lemma~\ref{lem: r3} to conclude that
\begin{equation*}
\E[R_2] \leq \bar R_2 + HK_1 \E\Big[\sum_{k=1}^{K_T}\frac{T_k}{\sqrt{t_k}}\Big] + 1,
\end{equation*}
where
\begin{equation*}
\bar R_2 := H\E\Big[\sum_{k=1}^{K_T}\sum_{t=t_k}^{t_{k+1}-1}\sum_{s'}\Big|\theta_*(s' | s_t, a_t) - \theta_k(s'|s_t, a_t)\Big|\Big].
\end{equation*}
$\bar R_2$ is the dominating term in the final $\otil(T^{2/3})$ regret bound and can be bounded by $H + 12HK_2(|\calS||\calA|T)^{2/3}$ where $K_2 := K_2(|\calS|, |\calA|, |O|, \iota)$ is the constant in Assumption~\ref{ass: concentration}. The detailed proof can be found in Lemma \ref{lem: r2bar}. However, we sketch the main steps of the proof here. By Assumption~\ref{ass: concentration}, one can show that
\begin{equation*}
\bar R_2 \leq \otil \Big(\E\Big[\sum_{t=1}^{T}\frac{HK_2}{\sqrt{\max\{1, {\tilde m}_{t}(s_t, a_t)\}}}\Big]\Big).
\end{equation*}
Now, let $E_2$ be the event that $\tilde m_t(s, a) \geq \alpha n_t(s, a)$ for all $s, a$. Note that by Lemma~\ref{lem: pseudo count vs true count general case} and union bound, $\p(E_2^c) \leq |\calS||\calA|\alpha$. Thus,
\begin{align*}
&\bar R_2 \leq \otil \Big(\E\Big[\sum_{t=1}^{T}\frac{HK_2}{\sqrt{\max\{1, {\tilde m}_{t}(s_t, a_t)\}}}\big(\1(E_2) + \1(E_2^c)\big)\Big]\Big) \\
&\leq \otil \Big( H\E\Big[\sum_{t=1}^{T}\frac{K_2}{\sqrt{\alpha\max\{1, {n}_{t}(s_t, a_t)\}}}\Big] + HK_2 |\calS||\calA|T\alpha\Big)
\end{align*}
Algebraic manipulation of the inner summation yields $\bar R_2 \leq \otil \Big( HK_2\sqrt \frac{|\calS||\calA|T}{\alpha} + HK_2 |\calS||\calA|T\alpha\Big)$.

Optimizing over $\alpha$ implies $\bar R_2 = \otil(HK_2(|\calS||\calA|T)^{2/3})$. 
Substituting  upper bounds for $\E[R_1], \E[R_2]$ and $\E[R_3]$, we get
\begin{align*}
\E[R_T] &= H\E[K_T] + \E[R_1] + \E[R_2] + \E[R_3] \\
&\leq (1 + H)\E[K_T]  + 12HK_2(|\calS||\calA|T)^{2/3} + (H+1)K_1 \E\Big[\sum_{k=1}^{K_T}\frac{T_k}{\sqrt{t_k}}\Big] + 2 + H.
\end{align*}
From Lemma~\ref{lem: number of episodes}, we know that $\E[K_T] = \otil(\sqrt{|\calS||\calA|T})$ and  $\sum_{k=1}^{K_T}\frac{T_k}{\sqrt{t_k}} = \otil(|\calS||\calA|\sqrt{T})$. Therefore, $\E[R_T] \leq \otil(HK_2(|\calS||\calA|T)^{2/3}).$

\section{Conclusions}

In this paper, we have presented one of the first online reinforcement learning  algorithms for POMDPs. Solving POMDPs is a hard problem. Designing an efficient learning algorithm that achieves sublinear regret is even harder. We show that the proposed \texttt{PSRL-POMDP} algorithm achieves a Bayesian regret bound of $\order(\log T)$ when the parameter is finite. When the parameter set may be uncountable, we showed a $\otil (T^{2/3})$ regret bound under two technical assumptions on the belief state approximation and transition kernel estimation. 
There has been recent work that does approximate belief state computation, as well as estimates transition dynamics of continuous MDPs, and in future work, we will try to incorporate such estimators. We also assume that the observation kernel is known. Note that without it, it is very challenging to design online learning algorithms for POMDPs. Posterior sampling-based algorithms in general are known to have superior numerical performance as compared to OFU-based algorithms for bandits and MDPs. In future work, we will also do an experimental investigation of the proposed algorithm. An impediment is that available POMDP solvers mostly provide approximate solutions which would lead to linear regret. In the future, we will also try to improve the regret for the general case to $\otil(\sqrt{T})$.

\bibliographystyle{plainnat}
\bibliography{online_rl}


\newpage
\appendix

\section{Regret Decomposition}
\label{app: aux}

\begin{lemma}
\label{lem: decomposition}
$R_T$ can be decomposed as $R_T = H\Ec[K_T] + R_1 + R_2 + R_3$, where
\begin{align*}
R_1 &:= \Ec\sbr{\sum_{k=1}^{K_T}T_k\Big[J(\theta_k) - J(\theta_*)\Big]}, \\
R_2 &:= H\Ec\sbr{\sum_{k=1}^{K_T}\sum_{t=t_k}^{t_{k+1}-1}\sbr{\sum_{s'}\abr{\theta_*(s'|s_t, a_t) - \theta_k(s'|s_t, a_t)} + \sum_s \abr{h_t(s;\theta_*) - h_t(s;\theta_k)}}}, \\
R_3 &:= \Ec\sbr{\sum_{k=1}^{K_T}\sum_{t=t_k}^{t_{k+1}-1}\Big[c(h_t(\cdot;\theta_*), a_t) - c(h_t(\cdot;\theta_k), a_t)\Big]}.
\end{align*}
\end{lemma}
\begin{proof}
First, note that $\Ec[C(s_t, a_t)| \calF_{t+}] = c(h_t(\cdot; \theta_*), a_t)$ for any $t \geq 1$. Thus, we can write:
\begin{align*}
R_T &= \Ec\Big[\sum_{t=1}^T\Big[C(s_t, a_t) - J(\theta_*)\Big]\Big] = \Ec\Big[\sum_{t=1}^T\Big[c(h_t(\cdot; \theta_*), a_t) - J(\theta_*)\Big]\Big].
\end{align*}
During episode $k$, by the Bellman equation for the sampled POMDP $\theta_k$ and that $a_t = \pi^*(h_t(\cdot;\theta_k);\theta_k)$, we can write:
\begin{align*}
&c(h_t(\cdot;\theta_k), a_t) - J(\theta_k) = v(h_t(\cdot;\theta_k);\theta_k) - \sum_{o}P(o|h_t(\cdot;\theta_k), a_t;\theta_k)v(h';\theta_k),
\end{align*}
where $h' = \tau(h_t(\cdot;\theta_k), a_t, o;\theta_k)$.
Using this equation, we proceed by decomposing the regret as
\begin{align*}
R_T &= \Ec\Big[\sum_{t=1}^T\Big[c(h_t(\cdot;\theta_*), a_t) - J(\theta_*)\Big]\Big] \\
&= \Ec\Big[\sum_{k=1}^{K_T}\sum_{t=t_k}^{t_{k+1}-1}\Big[c(h_t(\cdot;\theta_*), a_t) - J(\theta_*)\Big]\Big] \\
&= \Ec\Big[\sum_{k=1}^{K_T}\underbrace{\sum_{t=t_k}^{t_{k+1}-1}\Big[v(h_t(\cdot;\theta_k); \theta_k) - v(h_{t+1}(\cdot;\theta_k); \theta_k)\Big]}_{\text{telescopic sum}}\Big] + \underbrace{\Ec\sbr{\sum_{k=1}^{K_T}T_k\Big[J(\theta_k) - J(\theta_*)\Big]}}_{=: R_1} \\
&\qquad + \underbrace{\Ec\Big[\sum_{k=1}^{K_T}\sum_{t=t_k}^{t_{k+1}-1}\Big[v(h_{t+1}(\cdot;\theta_k); \theta_k) - \sum_{o \in O}P(o | h_t(\cdot;\theta_k), a_t; \theta_k)v(h'; \theta_k)\Big]\Big]}_{=: R_2'} \\
&\qquad + \underbrace{\Ec\Big[\sum_{k=1}^{K_T}\sum_{t=t_k}^{t_{k+1}-1}\Big[c(h_t(\cdot;\theta_*), a_t) - c(h_t(\cdot;\theta_k), a_t)\Big]\Big]}_{=: R_3}
\end{align*}
where $K_T$ is the number of episodes upto time $T$, $t_k$ is the start time of episode $k$ (we let $t_k = T+1$ for all $k > K_T$). The telescopic sum is equal to $v(h_{t_k}(\cdot;\theta_k); \theta_k) - v(h_{t_{k+1}}(\cdot;\theta_k); \theta_k) \leq H$. Thus, the first term on the right hand side is upper bounded by $H\Ec[K_T]$. Suffices to show that $R_2' \leq R_2$. Throughout the proof, we change the order of expectation and summation at several points. A rigorous proof for why this is allowed in the case that $K_T$ and $t_k$ are random variables is presented in the proof of Lemma~\ref{lem: r3}. 

We proceed by bounding the term $R_2'$. Recall that $h' = \tau(h_t(\cdot;\theta_k), a_t, o; \theta_k)$ and $h_{t+1}(\cdot; \theta_k) = \tau(h_t(\cdot;\theta_k), a_t, o_{t+1}; \theta_k)$.
Conditioned on $\calF_{t}, \theta_*, \theta_k$, the only random variable in $h_{t+1}(\cdot;\theta_k)$ is $o_{t+1}$ ($a_t=\pi^*(h_t(\cdot;\theta_k); \theta_k)$ is measurable with respect to the sigma algebra generated by $\calF_t, \theta_k$). Therefore,
\begin{align}
\label{eq: app expected v}
&\Ec\Big[v(h_{t+1}(\cdot;\theta_k); \theta_k) | \calF_{t}, \theta_k\Big] = \sum_{o \in O} v(h'; \theta_k) \pc(o_{t+1}=o | \calF_{t}, \theta_k).
\end{align}
We claim that $\pc(o_{t+1}=o | \calF_{t}, \theta_k) = P(o | h_t(\cdot;\theta_*), a_t; \theta_*)$: by the total law of probability and that $\pc (o_{t+1}=o | s_{t+1}=s', \calF_{t}, \theta_k) = \eta(o|s')$, we can write
\begin{align*}
&\pc(o_{t+1}=o | \calF_{t}, \theta_k) = \sum_{s'} \eta(o|s') \pc(s_{t+1}=s' | \calF_{t}, \theta_k).
\end{align*}
Note that
\begin{align*}
\pc(s_{t+1}=s' | \calF_{t}, \theta_k) &= \sum_{s} \pc(s_{t+1}=s' | s_t=s, \calF_t, a_t, \theta_k)\pc(s_t=s | \calF_t, \theta_k) \\
&=\sum_{s} \theta_*(s' | s, a_t) \pc(s_t=s | \calF_t).
\end{align*}
Thus,
\begin{align}
\label{eq: app p equal to p}
\pc(o_{t+1}=o | \calF_{t}, \theta_k) &= \sum_{s, s'} \eta(o | s') \theta_*(s' | s, a_t) h_t(s; \theta_*) = P(o | h_t(\cdot;\theta_*), a_t; \theta_*).
\end{align}
Combining \eqref{eq: app p equal to p} with \eqref{eq: app expected v} and substituting into $R_2'$, we get
\begin{align*}
R_2' &= \Ec\Big[\sum_{k=1}^{K_T}\sum_{t=t_k}^{t_{k+1}-1}\Big[\sum_{o \in O}\Big(P(o | h_t(\cdot; \theta_*), a_t; \theta_*) - P(o | h_t(\cdot;\theta_k), a_t; \theta_k)\Big)v(h'; \theta_k)\Big]\Big].
\end{align*}
Recall that for any $\theta \in \Theta$, $P(o | h_t(\cdot;\theta), a_t; \theta) = \sum_{s'} \eta(o | s') \sum_{s} \theta(s' | s, a_t) h_t(s; \theta)$. Thus,
\begin{align}
\label{eq: app r1bar in terms of h minus h}
R_2' &= \Ec\Big[\sum_{k=1}^{K_T}\sum_{t=t_k}^{t_{k+1}-1}\sum_{o, s'}v(h'; \theta_k) \eta(o | s')\sum_{s} \theta_*(s' | s, a_t) h_t(s; \theta_*)\Big] \nonumber \\
&\qquad- \Ec\Big[\sum_{k=1}^{K_T}\sum_{t=t_k}^{t_{k+1}-1}\sum_{o, s'}v(h'; \theta_k) \eta(o | s')\sum_{s} \theta_k(s' | s, a_t) h_t(s; \theta_*)\Big] \nonumber \\
&\qquad + \Ec\Big[\sum_{k=1}^{K_T}\sum_{t=t_k}^{t_{k+1}-1}\sum_{o, s'}v(h'; \theta_k) \eta(o | s')\sum_{s}\theta_k(s' | s, a_t) \big(h_t(s; \theta_*) -h_t(s; \theta_k)\big)\Big].
\end{align}
For the first term, note that conditioned on $\calF_{t}, \theta_*$, the distribution of $s_t$ is $h_t(\cdot;\theta_*)$ by the definition of $h_t$. Furthermore, $a_t$ is measurable with respect to the sigma algebra generated by $\calF_t, \theta_k$ since $a_t = \pi^*(h_t(\cdot;\theta_k); \theta_k)$. Thus, we have
\begin{align}
\label{eq: app removing ht theta star}
&\Ec\Big[v(h'; \theta_k)\sum_{s} \theta_*(s' | s, a_t) h_t(s; \theta_*)\Big|\calF_{t}, \theta_k\Big] = v(h'; \theta_k)\Ec\Big[\theta_*(s'|s_t, a_t)\Big|\calF_{t}, \theta_k\Big].
\end{align}
Similarly, for the second term on the right hand side of \eqref{eq: app r1bar in terms of h minus h}, we have
\begin{align}
\label{eq: app removing ht theta k}
&\Ec\Big[v(h'; \theta_k)\sum_{s} \theta_k(s' | s, a_t) h_t(s; \theta_*)\Big|\calF_{t}, \theta_k \Big] = v(h'; \theta_k)\Ec\Big[\theta_k(s'|s_t, a_t)\Big|\calF_{t}, \theta_k\Big].
\end{align}
Replacing \eqref{eq: app removing ht theta star}, \eqref{eq: app removing ht theta k} into \eqref{eq: app r1bar in terms of h minus h} and using the tower property of conditional expectation, we get
\begin{align}
\label{eq: app intermediate bound of r1bar}
R_2' &= \Ec\Big[\sum_{k=1}^{K_T}\sum_{t=t_k}^{t_{k+1}-1}\Big[\sum_{s'}\sum_{o} v(h'; \theta_k) \eta(o | s')\Big(\theta_*(s' | s_t, a_t) - \theta_k(s'|s_t, a_t)\Big)\Big]\Big] \nonumber \\
&\qquad+ \Ec\Big[\sum_{k=1}^{K_T}\sum_{t=t_k}^{t_{k+1}-1}\Big[\sum_{s'}\sum_{o}v(h'; \theta_k) \eta(o | s')\sum_{s} \theta_k(s' | s, a_t) \big(h_t(s; \theta_*) -h_t(s; \theta_k)\big)\Big]\Big].
\end{align}
Since $\sup_{b \in \Delta_\calS}v(b, \theta_k) \leq H$ and $\sum_o \eta(o|s') = 1$, the inner summation for the first term on the right hand side of \eqref{eq: app intermediate bound of r1bar} can be bounded as
\begin{align}
\label{eq: app common lem 1 tmp1}
&\sum_{o \in O}v(h'; \theta_k) \eta(o | s')\Big(\theta_*(s' | s_t, a_t) - \theta_k(s'|s_t, a_t)\Big) \leq H\Big|\theta_*(s' | s_t, a_t) - \theta_k(s'|s_t, a_t)\Big|.
\end{align}
Using $\sup_{b \in \Delta_\calS}v(b, \theta_k) \leq H$, $\sum_o \eta(o|s') = 1$ and $\sum_{s'}\theta_k(s'|s, a_t) = 1$, the second term on the right hand side of \eqref{eq: app intermediate bound of r1bar} can be bounded as
\begin{align}
\label{eq: app common lem 1 tmp2}
&\sum_{s'}\sum_{o \in O}v(h'; \theta_k) \eta(o | s')\sum_{s} \theta_k(s' | s, a_t)\big|h_t(s;\theta_*) - h_t(s;\theta_k)\big| \leq H\sum_{s}\big|h_t(s;\theta_*) - h_t(s;\theta_k)\big| \nonumber \\
\end{align}
Substituting \eqref{eq: app common lem 1 tmp1} and \eqref{eq: app common lem 1 tmp2} into \eqref{eq: app intermediate bound of r1bar} proves that $R_2' \leq R_2$.
\end{proof}


\section{Proofs of Section~\ref{sec: finite}}
\label{app: finite}

\subsection{Proof of Lemma~\ref{lem: 1-ft}}
\textbf{Lemma} (restatement of Lemma~\ref{lem: 1-ft})\textbf{.} Suppose Assumption~\ref{ass: kl divergence} holds. Then, there exist constants $\alpha > 1$ and $\beta > 0$ such that
\begin{align*}
\E[1 - f_t(\theta_*) | \theta_*] \leq \alpha \exp(-\beta t).
\end{align*}
\begin{proof}
Let $\tau_t$ be the trajectory $\{a_1, o_1, \cdots, a_{t-1}, o_{t-1}, o_t\}$ and define the likelihood function
\begin{align}
\calL(\tau_t | \theta) &:= \p(\tau_t | \theta) = \p(o_1|\theta)\prod_{\tau=2}^t \p(o_{\tau}|o_{1:\tau-1}, a_{1:\tau-1}|\theta) = \p(o_1|\theta)\prod_{\tau=2}^t \vthetatau(o_\tau)
\end{align}
Note that $\p(o_1|\theta)$ is independent of $\theta$, thus
\begin{align*}
\frac{\calL(\tau_t | \theta)}{\calL(\tau_t | \gamma)} = \prod_{\tau=2}^t\frac{\vthetatau(o_\tau)}{\vgammatau(o_\tau)}
\end{align*}
Recall that $f_t(\cdot)$ is the posterior associated with the likelihood given by
\begin{align*}
f_t(\theta) = \frac{\calL(\tau_t | \theta) f(\theta)}{\sum_{\gamma \in \Theta}\calL(\tau_t | \gamma)f(\gamma)}.
\end{align*}
We now proceed to lower bound $f_t(\theta_*)$. We can write
\begin{align*}
 f_t(\theta_*) &= \frac{\calL(\tau_t | \theta_*) f(\theta_*)}{\sum_{\theta}\calL(\tau_t | \theta)f(\theta)} = \frac{1}{1 + \sum_{\theta \neq \theta_*}\frac{f(\theta)}{f(\theta_*)} \frac{\calL(\tau_t | \theta)}{\calL(\tau_t | \theta_*)}} \\
&= \frac{1}{1 + \sum_{\theta \neq \theta_*}\frac{f(\theta)}{f(\theta_*)} \exp(-\sum_{\tau=1}^t\log\Lambda_\tau^\theta)},
\end{align*}
where we define $\Lambda_1^\theta := 1$ and for $\tau \geq 2$,
\begin{align*}
\Lambda_\tau^\theta := \frac{\vthetastartau(o_\tau)}{\vthetatau(o_\tau)}
\end{align*}
Note that without loss of generality, we can assume that the denominator in the definition of $\Lambda_\tau^\theta$ is positive (otherwise, $\calL(\tau_t | \theta) = 0$ and can be excluded from the denominator of $ f_t(\theta_*)$) and thus $\Lambda_\tau^\theta$ is well-defined. 

Denote by $Z_t^\theta := \sum_{\tau=1}^t\log\Lambda_\tau^\theta$ and decompose it as $Z_t^\theta = M_t^\theta + A_t^\theta$ where
\begin{align*}
M_t^\theta &:= \sum_{\tau=1}^t\Big(\log\Lambda_\tau^\theta - \E\Big[\log\Lambda_\tau^\theta \big| \calF_{\tau-1}, \theta_*\Big]\Big), \\
A_t^\theta &:= \sum_{\tau=1}^t \E\Big[\log\Lambda_\tau^\theta\big| \calF_{\tau-1}, \theta_*\Big].
\end{align*}
Note that the terms inside the first summation constitute a martingale difference sequence with respect to the filtration $(\calF_\tau)_{\tau \geq 1}$ and conditional probability $\p(\cdot|\theta_*)$. Each term is bounded as $|\log\Lambda_\tau^\theta - \E[\log\Lambda_\tau^\theta | \calF_{\tau-1}, \theta_*]| \leq d$ for some $d > 0$.
The second term, $A_t^\theta$ can be lower bounded using Assumption~\ref{ass: kl divergence} as follows
\begin{align*}
\E\Big[\log\Lambda_\tau^\theta\big| \calF_{\tau-1}, \theta_*\Big] &= \E\bigg[\E\Big[\log\Lambda_\tau^\theta\big| \calF_{\tau-1}, a_{\tau-1}, \theta_*\Big]\Big| \calF_{\tau-1}, \theta_*\bigg] \\
&= \E\Big[\calK(\vthetastartau \| \vthetatau)\big|\calF_{\tau-1}, \theta_*\Big] \geq \epsilon
\end{align*}
Summing over $\tau$ implies that
\begin{align}
\label{eq: atdelta bound}
A_t^\theta \geq \epsilon t.
\end{align}
To bound $M_t^\theta$, let $0 < \delta < \epsilon$, and apply Azuma's inequality to obtain
\begin{align*}
\p\Big(|M_t^\theta| \geq \delta t \big| \theta_*\Big) \leq 2\exp(-\frac{\delta^2t}{2d^2}).
\end{align*}
Union bound over all $\theta \neq \theta_*$ implies that the event $B_t^\delta := \cap_{\theta \neq \theta_*}\{|M_t^\theta| \leq \delta t\}$ happens with probability at least $1 - 2(|\Theta|-1)\exp(-\frac{\delta^2t}{2d^2})$. If $B_t^\delta$ holds, then $-M_t^\theta \leq \delta t$ for all $\theta \neq \theta_*$. Combining this with \eqref{eq: atdelta bound} implies that $\exp(-M_t^\theta - A_t^\theta) \leq \exp(\delta t - \epsilon t)$. 
Therefore,
\begin{align*}
\E[f_t(\theta_*) | \theta_*] &= \E\Bigg[\frac{1}{1 + \sum_{\theta \neq \theta_*}\frac{f(\theta)}{f(\theta_*)} \exp(-M_t^\theta - A_t^\theta)}\bigg| \theta_* \Bigg] \\
&\geq \E\Bigg[\frac{\1(B_\delta^t)}{1 + \sum_{\theta \neq \theta_*}\frac{f(\theta)}{f(\theta_*)} \exp(\delta t - \epsilon t)} \bigg| \theta_* \Bigg] \\
&= \frac{\p(B_\delta^t | \theta_*)}{1 + \frac{1 - f(\theta_*)}{f(\theta_*)} \exp(\delta t - \epsilon t)} \\
&\geq \frac{1 - 2(|\Theta| - 1)\exp(-\frac{\delta^2t}{2d^2})}{1 + \frac{1 - f(\theta_*)}{f(\theta_*)} \exp(\delta t - \epsilon t)}.
\end{align*}
Now, by choosing $\delta = \epsilon/2$, and constants $\alpha = 2\max \{\max_{\theta \in \Theta}\frac{1 - f(\theta)}{f(\theta)}, 2(|\Theta|-1)\}$, and $\beta = \min \{\frac{\epsilon}{2}, \frac{\epsilon^2}{8d^2}\}$, we have
\begin{align*}
\E[1 - f_t(\theta_*) | \theta_*] &\leq 1 - \frac{1 - 2(|\Theta| - 1)\exp(-\frac{\delta^2t}{2d^2})}{1 + \frac{1 - f(\theta_*)}{f(\theta_*)} \exp(\delta t - \epsilon t)} \\
&= \frac{\frac{1 - f(\theta_*)}{f(\theta_*)} \exp(\delta t - \epsilon t) + 2(|\Theta| - 1)\exp(-\frac{\delta^2t}{2d^2})}{1 + \frac{1 - f(\theta_*)}{f(\theta_*)} \exp(\delta t - \epsilon t)} \\
&\leq \frac{1 - f(\theta_*)}{f(\theta_*)} \exp(\delta t - \epsilon t) + 2(|\Theta| - 1)\exp(-\frac{\delta^2t}{2d^2}) \\
&= \frac{1 - f(\theta_*)}{f(\theta_*)} \exp(-\frac{\epsilon t}{2}) + 2(|\Theta| - 1)\exp(-\frac{\epsilon^2t}{8d^2}) \\
&\leq \alpha \exp(-\beta t).
\end{align*}
\end{proof}

\section{Proofs of Section~\ref{sec: general}}
\label{app: general}

\subsection{Full Upper Bound on the Expected Regret of Theorem~\ref{thm: regret bound}}
\label{app: full upper bound}
The exact expression for the upper bound of the expected regret in Theorem~\ref{thm: regret bound} is
\begin{align*}
&\E[R_T] = H\E[K_T] + \E[R_1] + \E[R_2] + \E[R_3] \\
&\leq (1 + H)\E[K_T]  + 12HK_2(|\calS||\calA|T)^{2/3} \\
&\quad+ (H+1)K_1 \E\Big[\sum_{k=1}^{K_T}\frac{T_k}{\sqrt{t_k}}\Big] + 2 + H \\
&\leq (1 + H)\sqrt{2T(1 + |\calS||\calA| \log (T+1))}  + 12HK_2(|\calS||\calA|T)^{2/3} \\
&\quad+ 7(H+1)K_1\sqrt{2T}(1 + |\calS||\calA|\log (T+1)) \log\sqrt{2T} + 2 + H.
\end{align*}

\subsection{Finite-parameter Case Satisfies Assumptions~\ref{ass: concentrating belief} and \ref{ass: concentration}}
In this section, we show that Assumptions~\ref{ass: concentrating belief} and \ref{ass: concentration} are satisfied for the finite-parameter case i.e., $|\Theta| < \infty$ as long as the \texttt{PSRL-POMDP} generates a deterministic schedule. As an instance, a deterministic schedule can be generated by choosing $\tilde m_t(s, a) = t$ for all state-action pairs $(s, a)$ and running Algorithm~\ref{alg: posterior sampling} with either $\sched(t_k, T_{k-1}) = 2t_k$ or $\sched(t_k, T_{k-1}) = t_k + T_{k-1}$.
\begin{lemma}
\label{lem: assumptions satisfied for finite-parameter}
Assume $|\Theta| < \infty$. If Algorithm~ref{alg: posterior sampling} generates a deterministic schedule, then Assumptions~\ref{ass: concentrating belief} and \ref{ass: concentration} are satisfied.
\end{lemma}
\begin{proof}
Observe that the left hand side of \eqref{eq: ass concentrating belief} is zero if $\theta_{k(t)} = \theta_*$, and is upper bounded by $2$ if $\theta_{k(t)} \neq \theta_*$. Thus, we can write
\begin{align*}
\E\Big[\sum_{s}\big|h_t(s;\theta_*) - h_t(s;\theta_{k(t)})\big|\Big | \theta_*\Big] \leq 2\p(\theta_{k(t)} \neq \theta_* | \theta_*) = 2\E\sbr{1 - f_{t_{k(t)}}(\theta_*) | \theta_*} \leq \alpha \exp(-\beta t_{k(t)}),
\end{align*} 
which obviously satisfies Assumption~\ref{ass: concentrating belief} by choosing a large enough constant $K_1$. Here, the last equality is by Lemma~\ref{lem: 1-ft} and that the start time of episode $k(t)$ is deterministic.

To see why Assumption~\ref{ass: concentration} is satisfied, let $\hat \theta_t$ be the Maximum a Posteriori (MAP) estimator, i.e., $\hat \theta_t = \argmax_{\theta \in \Theta} f_t(\theta)$. Then, the left hand side of \eqref{eq: ass concentration} is equal to zero if $\hat \theta_t = \theta_*$. Note that this happens with high probability with the following argument:
\begin{align*}
\p(\hat \theta_t \neq \theta_* | \theta_*) \leq \p\rbr{f_t(\theta_*) \leq 0.5 | \theta_*} = \p(1 - f_t(\theta_*) \geq 0.5 | \theta_*) \leq 2\E[1 - f_t(\theta_*) | \theta_*] \leq 2\alpha \exp(-\beta t).
\end{align*}
Here the first inequality is by the fact that if $f_t(\theta_*) > 0.5$, then the MAP estimator would choose $\hat \theta_t = \theta_*$. The second inequality is by applying Markov inequality and the last inequality is by Lemma~\ref{lem: 1-ft}. Note that $\tilde m_t(s, a) \leq t$ by definition. We claim that Assumption~\ref{ass: concentration} is satisfied by choosing $K_2 = 2\sqrt{(-1/\beta)\log (\delta/2\alpha)}$. To see this, note that $2\alpha \exp(-\beta t) \leq \delta$ for $t \geq (-1/\beta)\log (\delta/2\alpha)$. In this case, \eqref{eq: ass concentration} automatically holds since with probability at least $1 - \delta$ the left hand side is zero. For $t  < (-1/\beta)\log (\delta/2\alpha)$, note that the left hand side of \eqref{eq: ass concentration} can be at most 2. Therefore, $K_2$ can be found by solving $2 \leq K_2/\sqrt{(-1/\beta)\log (\delta/2\alpha)}$.
\end{proof}

\subsection{Auxiliary Lemmas for Section~\ref{sec: general}}
\begin{lemma}
\label{lem: j minus jstar}[Lemma 3 in \cite{ouyang2017learning}]
The term $\E[R_1]$ can be bounded as $\E[R_1] \leq \E[K_T]$.
\end{lemma}
\begin{proof}
\begin{align*}
\E[R_1] &= \E\Big[\sum_{k=1}^{K_T}T_k\Big[J(\theta_k) - J(\theta_*)\Big]\Big] = \E\Big[\sum_{k=1}^{\infty}\1(t_k \leq T)T_kJ(\theta_k)\Big] - T\E[J(\theta_*)].
\end{align*}
By monotone convergence theorem and the fact that $J(\theta_k) \geq 0$ and $T_k \leq T_{k-1} + 1$ (the first criterion in determining the episode length in Algorithm~\ref{alg: posterior sampling}), the first term can be bounded as
\begin{align*}
&\E\Big[\sum_{k=1}^{\infty}\1(t_k \leq T)T_kJ(\theta_k)\Big] = \sum_{k=1}^{\infty}\E\Big[\1(t_k \leq T)T_kJ(\theta_k)\Big] \\
&\leq \sum_{k=1}^{\infty}\E\Big[\1(t_k \leq T)(T_{k-1}+1)J(\theta_k)\Big].
\end{align*}
Note that $\1(t_k \leq T)(T_{k-1} + 1)$ is $\calF_{t_k}$-measurable. Thus, by the property of posterior sampling (Lemma~\ref{lem: property of posterior sampling}), $\E[\1(t_k \leq T)(T_{k-1}+1)J(\theta_k)] = \E[\1(t_k \leq T)(T_{k-1}+1)J(\theta_*)]$. Therefore,
\begin{align*}
\E[R_1] &\leq \E\Big[\sum_{k=1}^{\infty}\1(t_k \leq T)(T_{k-1}+1)J(\theta_*)\Big] - T\E[J(\theta_*)] \\
&= \E\Big[J(\theta_*)(K_T + \sum_{k=1}^{K_T}T_{k-1})\Big] - T\E[J(\theta_*)] \\
&= \E[J(\theta_*)K_T] +  \E\Big[J(\theta_*)(\sum_{k=1}^{K_T}T_{k-1} - T)\Big]\leq \E[K_T],
\end{align*}
where the last inequality is by the fact that $\sum_{k=1}^{K_T}T_{k-1} - T \leq 0$ and $0 \leq J(\theta_*) \leq 1$.
\end{proof}

\begin{lemma}
\label{lem: r3}
The term $\E[R_3]$ can be bounded as
\begin{align*}
\E[R_3] \leq K_1\E\Big[\sum_{k=1}^{K_T}\frac{T_k}{\sqrt{t_k}}\Big] + 1,
\end{align*}
where $K_1:= K_1(|\calS|, |\calA|, |O|, \iota)$ is the constant in Assumption~\ref{ass: concentrating belief}.
\end{lemma}
\begin{proof}
Recall that
\begin{align*}
\E[R_3] &= \E\Big[\sum_{k=1}^{K_T}\sum_{t=t_k}^{t_{k+1}-1}\Big[c(h_t(\cdot;\theta_*), a_t) - c(h_t(\cdot;\theta_k), a_t)\Big]\Big].
\end{align*}
Let $k(t)$ be a random variable denoting the episode number at time $t$, i.e., $t_{k(t)} \leq t < t_{k(t) + 1}$ for all $t \leq T$.
By the definition of $c$, we can write
\begin{align*}
\E[R_3] &= \E\Big[\sum_{k=1}^{K_T}\sum_{t=t_k}^{t_{k+1}-1}\sum_{s}C(s, a_t)\Big[h_t(s;\theta_*) - h_t(s;\theta_k)\Big]\Big] \\
&=\E\Big[\sum_{t=1}^T \sum_{s}C(s, a_t)\Big[h_t(s;\theta_*) - h_t(s;\theta_{k(t)})\Big]\Big] \\
&\leq \sum_{t=1}^T \E\Big[\sum_{s}\big|h_t(s;\theta_*) - h_t(s;\theta_{k(t)})\big|\Big]  \\
&= \E\Big[\sum_{k=1}^{K_T}\sum_{t=t_k}^{t_{k+1}-1}\E\Big[\sum_{s}\big|h_t(s;\theta_*) - h_t(s;\theta_k)\big|\Big]\Big],
\end{align*}
where 
the inequality is by $0 \leq C(s, a_t) \leq 1$. Let $K_1:= K_1(|\calS|, |\calA|, |O|, \iota)$ be the constant in Assumption~\ref{ass: concentrating belief} and define event $E_1$ as the successful event of Assumption~\ref{ass: concentrating belief} where $\E\Big[\sum_{s}\big|h_t(s;\theta_*) - h_t(s;\theta_k)\big|\Big] \leq \frac{K_1}{\sqrt{t_k}}$ happens. We can write
\begin{align*}
&\E\Big[\sum_{s}\big|h_t(s;\theta_*) - h_t(s;\theta_k)\big|\Big] \\
&= \E\Big[\sum_{s}\big|h_t(s;\theta_*) - h_t(s;\theta_k)\big|\Big] (\1(E_1) + \1(E_1^c)) \\
&\leq \frac{K_1}{\sqrt{t_k}} + 2\1(E_1^c).
\end{align*}
Recall that by Assumption~\ref{ass: concentrating belief}, $\p(E_1^c) \leq \delta$. Therefore,
\begin{align*}
\E[R_3] \leq K_1\E\Big[\sum_{k=1}^{K_T}\frac{T_k}{\sqrt{t_k}}\Big] + 2T\delta.
\end{align*}
Choosing $\delta = \min(1/(2T), 1/(2HT))$ completes the proof.
\end{proof}

\begin{lemma}
\label{lem: r2bar}
The term $\bar R_2$ can be bounded as
\begin{align*}
\bar R_2 \leq H + 12HK_2(|\calS||\calA|T)^{2/3},
\end{align*}
where $K_2 := K_2(|\calS|, |\calA|, |O|, \iota)$ in Assumption~\ref{ass: concentration}.
\end{lemma}
\begin{proof}
Recall that
\begin{align}
\label{eq: app nice bound of r1bar}
\bar R_2 = H\E\Big[\sum_{k=1}^{K_T}\sum_{t=t_k}^{t_{k+1}-1}\sum_{s'}\Big|\theta_*(s' | s_t, a_t) - \theta_k(s'|s_t, a_t)\Big|\Big].
\end{align}
We proceed by bounding the inner term of the above equation. For notational simplicity, define $z := (s, a)$ and $z_t := (s_t, a_t)$. Let $\hat \theta_{t_k}$ be the estimator in Assumption~\ref{ass: concentration} and define the confidence set $B_k$ as
\begin{align*}
B_k &:= \Big\{\theta \in \Theta_H : \sum_{s' \in \calS} \Big|\theta(s'|z) - \hat \theta_k(s'|z)\Big| \leq \frac{K_2}{\sqrt{\max\{1, {\tilde m}_{t_k}(z)\}}}, \forall z \in \calS \times \calA\Big\},
\end{align*}
where $K_2 := K_2(|\calS|, |\calA|, |O|, \iota)$ is the constant in Assumption~\ref{ass: concentration}. Note that $B_k$ reduces to the confidence set used in \cite{jaksch2010near,ouyang2017learning} in the case of perfect observation by choosing $\tilde m_t(s, a) = n_t(s, a)$. By triangle inequality, the inner term in \eqref{eq: app nice bound of r1bar} can be bounded by
\begin{align*}
&\sum_{s'}\Big|\theta_*(s' | z_t) - \theta_k(s'|z_t)\Big| \\
&\leq \sum_{s'}\Big|\theta_*(s' | z_t) - \hat \theta_{t_k}(s'|z_t)\Big| + \sum_{s'}\Big|\theta_k(s' | z_t) - \hat \theta_{t_k}(s'|z_t)\Big| \\
&\leq 2\big(\1(\theta_* \notin B_k) + \1(\theta_k \notin B_k)\big) + \frac{2K_2}{\sqrt{\max\{1, {\tilde m}_{t_k}(z_t)\}}}.
\end{align*}
Substituting this into \eqref{eq: app nice bound of r1bar} implies
\begin{align}
\label{eq: app two terms r1bar}
\bar R_2 &\leq 2H\E\Big[\sum_{k=1}^{K_T}\sum_{t=t_k}^{t_{k+1}-1}\big(\1(\theta_* \notin B_k) + \1(\theta_k \notin B_k)\big)\Big] + 2H\E\Big[\sum_{k=1}^{K_T}\sum_{t=t_k}^{t_{k+1}-1}\frac{K_2}{\sqrt{\max\{1, {\tilde m}_{t_k}(z_t)\}}}\Big].
\end{align}
We need to bound these two terms separately. 
\paragraph{Bounding the first term.} For the first term we can write:
\begin{align*}
\E\Big[\sum_{k=1}^{K_T}\sum_{t=t_k}^{t_{k+1}-1}\big(\1(\theta_* \notin B_k) + \1(\theta_k \notin B_k)\big)\Big] &= \E\Big[\sum_{k=1}^{K_T}T_k\big(\1(\theta_* \notin B_k) + \1(\theta_k \notin B_k)\big)\Big] \\
&\leq T\E\Big[\sum_{k=1}^{K_T}\big(\1(\theta_* \notin B_k) + \1(\theta_k \notin B_k)\big)\Big] \\
&\leq T\sum_{k=1}^{T}\E\Big[\big(\1(\theta_* \notin B_k) + \1(\theta_k \notin B_k)\big)\Big],
\end{align*}
where the last inequality is by the fact that $K_T \leq T$. Now, observe that since $B_k$ is $\calF_{t_k}$-measurable, Lemma \ref{lem: property of posterior sampling} implies that $\E[\1(\theta_k \notin B_k)] = \E[\1(\theta_* \notin B_k)]$. Moreover, by Assumption~\ref{ass: concentration}, $\E[\1(\theta_* \notin B_k)] = \p(\theta_* \notin B_k) \leq \delta$. By choosing $\delta = \frac{1}{4T^2}$, we get
\begin{align}
\label{eq: app bound on the first term}
\E\Big[\sum_{k=1}^{K_T}\sum_{t=t_k}^{t_{k+1}-1}\big(\1(\theta_* \notin B_k) + \1(\theta_k \notin B_k)\big)\Big] \leq \frac{1}{2}.
\end{align} 
\paragraph{Bounding the second term.} To bound the second term of \eqref{eq: app two terms r1bar}, observe that by the second criterion of the algorithm in choosing the episode length, we have $2\tilde m_{t_k}(z_t) \geq \tilde m_{t}(z_t)$. Thus,
\begin{align}
\label{eq: app nt and alpha}
&\E\Big[\sum_{k=1}^{K_T}\sum_{t=t_k}^{t_{k+1}-1}\frac{K_2}{\sqrt{\max\{1, {\tilde m}_{t_k}(z_t)\}}}\Big] \leq \E\Big[\sum_{t=1}^{T}\frac{\sqrt 2K_2}{\sqrt{\max\{1, {\tilde m}_{t}(z_t)\}}}\Big] \nonumber \\
&= \sum_{t=1}^{T} \sum_{z}\E\Big[\frac{\sqrt 2K_2 \1(z_t=z)}{\sqrt{\max\{1, {\tilde m}_{t}(z)\}}}\Big] \nonumber \\
&= \sum_{t=1}^{T} \sum_{z}\E\Big[\frac{\sqrt 2K_2 \1(z_t=z)}{\sqrt{\max\{1, {\tilde m}_{t}(z)\}}}\1\big({\tilde m}_{t}(z) \geq \alpha n_t(z)\big)\Big] \nonumber \\
&\qquad\qquad+ \sum_{t=1}^{T} \sum_{z}\E\Big[\frac{\sqrt 2K_2 \1(z_t=z)}{\sqrt{\max\{1, {\tilde m}_{t}(z)\}}}\1\big({\tilde m}_{t}(z) < \alpha n_t(z)\big)\Big] \nonumber \\
&\leq \sum_{t=1}^{T} \sum_{z}\E\Big[\frac{\sqrt 2K_2 \1(z_t=z)}{\sqrt{\max\{1, \alpha{n}_{t}(z)\}}}\Big] + \sum_{t=1}^{T} \sum_{z}\E\Big[\sqrt 2 K_2 \1\big({\tilde m}_{t}(z) < \alpha n_t(z)\big)\Big].
\end{align}
Lemma~\ref{lem: pseudo count vs true count general case} implies that $\E\Big[\1\big({\tilde m}_{t}(z) < \alpha n_t(z)\big)\Big] = \p({\tilde m}_{t}(z) < \alpha n_t(z)) \leq \alpha$. Thus, the second term in \eqref{eq: app nt and alpha} can be bounded by $\sqrt 2K_2|\calS||\calA|T\alpha$. To bound the first term of \eqref{eq: app nt and alpha}, we can write:
\begin{align*}
&\sum_{t=1}^{T} \sum_{z}\E\Big[\frac{\sqrt 2K_2 \1(z_t=z)}{\sqrt{\max\{1, \alpha{n}_{t}(z)\}}}\Big] \\
&\leq \sqrt \frac{2}{\alpha}K_2 \E\Big[\sum_{z}\sum_{t=1}^{T}\frac{\1(z_t=z)}{\sqrt{\max\{1, {n}_{t}(z)\}}}\Big].
\end{align*}
Observe that whenever $z_t=z$, $n_t(z)$ increases by $1$. Since, $n_t(z)$ is the number of visits to $z$ by time $t-1$ (including $t-1$ and excluding $t$), the denominator will be $1$ for the first two times that $z_t=z$. Therefore, the term inside the expectation can be bounded by
\begin{align*}
\sum_{z}\sum_{t=1}^{T}\frac{\1(z_t=z)}{\sqrt{\max\{1, {n}_{t}(z)\}}} &= \sum_{z} \E\Big[\1(n_{T+1}(z)>0) + \sum_{j=1}^{n_{T+1}(z)-1}\frac{1}{\sqrt j}\Big] \\
&\leq \sum_{z} \E\Big[\1(n_{T+1}(z)>0) + 2\sqrt{n_{T+1}(z)}\Big]\\
& \leq 3\sum_z\sqrt{n_{T+1}(z)}.
\end{align*}
Since $\sum_z n_{T+1}(z) = T$, Cauchy Schwartz inequality implies
\begin{align*}
3\sum_z\sqrt{n_{T+1}(z)} \leq 3\sqrt{|\calS||\calA|\sum_{z}n_{T+1}(z)} = 3\sqrt{|\calS||\calA|T}.
\end{align*}
Therefore, the first term of \eqref{eq: app nt and alpha} can be bounded by
\begin{align*}
\sum_{t=1}^{T} \sum_{z}\E\Big[\frac{\sqrt 2K_2 \1(z_t=z)}{\sqrt{\max\{1, \alpha{n}_{t}(z)\}}}\Big] \leq 3K_2 \sqrt{\frac{2|\calS||\calA|T}{\alpha}}.
\end{align*}
Substituting this bound in \eqref{eq: app nt and alpha} along with the bound on the second term of \eqref{eq: app nt and alpha}, we obtain
\begin{align*}
&\E\Big[\sum_{k=1}^{K_T}\sum_{t=t_k}^{t_{k+1}-1}\frac{K_2}{\sqrt{\max\{1, {\tilde m}_{t_k}(z_t)\}}}\Big] \leq 3K_2 \sqrt{\frac{2|\calS||\calA|T}{\alpha}} + \sqrt 2K_2|\calS||\calA|T\alpha.
\end{align*}
$\alpha = (3/2)^{2/3}(|\calS||\calA|T)^{-1/3}$ minimizes the upper bound, and thus
\begin{align}
\label{eq: app bound on the second term}
\E\Big[\sum_{k=1}^{K_T}\sum_{t=t_k}^{t_{k+1}-1}&\frac{K_2}{\sqrt{\max\{1, {\tilde m}_{t_k}(z_t)\}}}\Big] \leq 6K_2(|\calS||\calA|T)^{2/3}.
\end{align}
By substituting \eqref{eq: app bound on the first term} and \eqref{eq: app bound on the second term} into \eqref{eq: app two terms r1bar}, we get
\begin{align*}
\bar R_2 \leq H + 12HK_2(|\calS||\calA|T)^{2/3}.
\end{align*}
\end{proof}

\begin{lemma}
\label{lem: number of episodes}
The following inequalities hold:
\begin{enumerate}
\item The number of episodes $K_T$ can be bounded as $K_T \leq \sqrt{2T(1 + |\calS||\calA| \log (T+1))} = \otil(\sqrt{|\calS||\calA|T})$.
\item The following inequality holds: $\sum_{k=1}^{K_T}\frac{T_k}{\sqrt{t_k}} \leq 7\sqrt{2T}(1 + |\calS||\calA|\log (T+1)) \log\sqrt{2T} = \otil(|\calS||\calA|\sqrt{T})$.
\end{enumerate}
\end{lemma}
\begin{proof}
We first provide an intuition why these results should be true. Note that the length of the episodes is determined by two criteria. The first criterion triggers when $T_k = T_{k-1} + 1$ and the second criterion triggers when the pseudo counts doubles for a state-action pair compared to the beginning of the episode. Intuitively speaking, the second criterion should only happen logarithmically, while the first criterion occurs more frequently. This means that one could just consider the first criterion for an intuitive argument. Thus, if we ignore the second criterion, we get $T_k = \order(k)$, $K_T = \order(\sqrt{T})$, and $t_k = \order(k^2)$ which implies $\sum_{k=1}^{K_T}\frac{T_k}{\sqrt{t_k}} = \order(K_T) = \order(\sqrt T)$. The rigorous proof is stated in the following.
 
1. Define macro episodes with start times $t_{m_i}$ given by $t_{m_1} = t_1$ and $t_{m_i} :=$
\begin{align*}
\min\{t_k > t_{m_{i-1}} : \tilde m_{t_k}(s, a) > 2\tilde m_{t_{k-1}}(s, a) \text{ for some } (s, a)\}.
\end{align*}
Note that a new macro episode starts when the second criterion of episode length in Algorithm~\ref{alg: posterior sampling} triggers. Let $M_T$ be the random variable denoting the number of macro episodes by time $T$ and define $m_{M_T+1} = K_{T}+1$.

Let $\tilde T_i$ denote the length of macro episode $i$. Note that $\tilde T_i = \sum_{k=m_i}^{m_{i+1}-1}T_k$. Moreover, from the definition of macro episodes, we know that all the episodes in a macro episode except the last one are triggered by the first criterion, i.e., $T_k = T_{k-1}+1$ for all $m_i \leq k \leq m_{i+1}-2$. This implies that
\begin{align*}
&\tilde T_i = \sum_{k=m_i}^{m_{i+1}-1}T_k = T_{m_{i+1}-1} + \sum_{j=1}^{m_{i+1} - m_i - 1} (T_{m_{i}-1} + j) \\
&\geq 1 + \sum_{j=1}^{m_{i+1} - m_i - 1}(1 + j) = \frac{(m_{i+1} - m_i)(m_{i+1} - m_i+1)}{2}.
\end{align*}
This implies that $m_{i+1} - m_i \leq \sqrt{2\tilde T_i}$. Now, we can write:
\begin{align}
\label{eq: app kt and mt}
&K_T = m_{M_T+1} - 1 = \sum_{i=1}^{M_T}(m_{i+1} - m_i) \nonumber \\
&\leq \sum_{i=1}^{M_T}\sqrt{2 \tilde T_i} \leq \sqrt{2M_T\sum_{i}\tilde T_i} = \sqrt{2M_TT},
\end{align}
where the last inequality is by Cauchy-Schwartz.

Now, it suffices to show that $M_T \leq 1 + |\calS||\calA| \log (T+1)$. Let $\calT_{s, a}$ be the start times at which the second criterion is triggered at state-action pair $(s, a)$, i.e.,
\begin{align*}
\calT_{s, a} := \{t_k \leq T : \tilde m_{t_k}(s, a) > 2 \tilde m_{t_{k-1}}(s, a)\}.
\end{align*}
We claim that $|\calT_{s, a}| \leq \log(\tilde m_{T+1}(s, a))$. To prove this claim, assume by contradiction that $|\calT_{s, a}| \geq \log (\tilde m_{T+1}(s, a))+1$, then
\begin{align*}
&\tilde m_{t_{K_T}}(s, a) \geq \prod_{t_k \leq T, \tilde m_{t_{k-1}}(s, a) \geq 1}\frac{\tilde m_{t_{k}}(s, a)}{\tilde m_{t_{k-1}}(s, a)} \\
&\geq \prod_{t_k \in \calT_{s, a}, \tilde m_{t_{k-1}}(s, a) \geq 1}\frac{\tilde m_{t_{k}}(s, a)}{\tilde m_{t_{k-1}}(s, a)} \\
&> \prod_{t_k \in \calT_{s, a}, \tilde m_{t_{k-1}}(s, a) \geq 1} 2  = 2^{|\calT_{s, a}|-1}\geq \tilde m_{T+1}(s, a),
\end{align*}
which is a contradiction. The second inequality is by the fact that $\tilde m_t(s, a)$ is non-decreasing, and the third inequality is by the definition of $\calT_{s, a}$.
Therefore,
\begin{align}
\label{eq: app bound on mt}
&M_T \leq 1 + \sum_{s, a}|\calT_{s, a}| \leq 1 + \sum_{s, a}\log(\tilde m_{T+1}(s, a)) \nonumber \\
&\leq 1 + |\calS||\calA| \log(\sum_{s, a}\tilde m_{T+1}(s, a)/|\calS||\calA|) \nonumber \\
&= 1 + |\calS||\calA|\log (T+1),
\end{align}
where the third inequality is due to the concavity of $\log$ and the last inequality is by the fact that $\tilde m_{T+1}(s, a) \leq T+1$.

2. First, we claim that $T_k \leq \sqrt{2T}$ for all $k \leq K_T$. To see this, assume by contradiction that $T_{k^*} > \sqrt{2T}$ for some $k^* \leq K_T$. By the first stopping criterion, we can conclude that $T_{k^*-1} > \sqrt{2T}-1$, $T_{k^*-2} > \sqrt{2T}-2$, \dots, $T_1 > \max\{\sqrt{2T} - k^* + 1, 0\}$ since the episode length can increase at most by one compared to the previous one. Note that $k^* \geq \sqrt{2T}-1$, because otherwise $T_1 > 2$ which is not feasible since $T_1 \leq T_0+1 = 2$.
Thus, $\sum_{k=1}^{k^*}T_k > 0.5\sqrt{2T}(\sqrt{2T}+1) > T$ which is a contradiction.

We now proceed to lower bound $t_k$. By the definition of macro episodes in part (1), during a macro episode length of the episodes except the last one are determined by the first criterion, i.e., for macro episode $i$, one can write $T_k = T_{k-1}+1$ for $m_i \leq k \leq m_{i+1}-2$. Hence, for $m_i \leq k \leq m_{i+1}-2$,
\begin{align*}
t_{k+1} &= t_k + T_k = t_k + T_{m_i-1} + k - (m_i-1) \\
&\geq t_k + k - m_i + 1.
\end{align*}
Recursive substitution of $t_k$ implies that $t_{k} \geq t_{m_i} + 0.5(k-m_i)(k-m_i+1)$ for $m_i \leq k \leq m_{i+1}-1$. Thus,
\begin{align}
\label{eq: app sum of Tk divided by tk}
&\sum_{k=1}^{K_T}\frac{T_k}{\sqrt{t_k}} \leq \sqrt{2T} \sum_{i=1}^{M_T}\sum_{k=m_i}^{m_{i+1}-1}\frac{1}{\sqrt{t_k}} \nonumber \\
&\leq \sqrt{2T} \sum_{i=1}^{M_T}\sum_{k=m_i}^{m_{i+1}-1}\frac{1}{\sqrt{t_{m_i} + 0.5(k-m_i)(k-m_i+1)}}.
\end{align}
The denominator of the summands at $k=m_i$ is equal to $\sqrt{t_{m_i}}$. For other values of $k$ it can be lower bounded by $0.5(k-m_i)^2$. Thus,
\begin{align*}
&\sum_{i=1}^{M_T}\sum_{k=m_i}^{m_{i+1}-1}\frac{1}{\sqrt{t_{m_i} + 0.5(k-m_i)(k-m_i+1)}} \\
&\leq \sum_{i=1}^{M_T} \frac{1}{\sqrt{t_{m_i}}} + \sum_{i=1}^{M_T}\sum_{k=m_i+1}^{m_{i+1}-1} \frac{\sqrt 2}{k-m_i} \\
&\leq M_T + \sum_{i=1}^{M_T}\sum_{j=1}^{m_{i+1}-m_i-1} \frac{\sqrt 2}{j} \\
&\leq M_T + \sqrt 2 (M_T + \sum_{i=1}^{M_T}\log (m_{i+1}-m_i)) \\
&\leq M_T(1 + \sqrt 2) + \sqrt 2M_T \log(\frac{1}{M_T}\sum_{i=1}^{M_T}(m_{i+1}-m_i)) \\
&\leq M_T(1 + \sqrt 2) + \sqrt 2M_T \log\sqrt{2T} \\
&\leq 7M_T\log\sqrt{2T},
\end{align*}
where the second inequality is by $t_{m_i} \geq 1$, the third inequality is by the fact that $\sum_{j=1}^K 1/j \leq 1 + \int_1^K dx/x = 1 + \log K$, the forth inequality is by concavity of $\log$ and the fifth inequality is by the fact that $\sum_{i=1}^{M_T}(m_{i+1}-m_i) = m_{M_T+1} - 1 = K_T$ and $K_T/M_T \leq \sqrt{2T/M_T} \leq \sqrt{2T}$ (see \eqref{eq: app kt and mt}). Substituting this bound into \eqref{eq: app sum of Tk divided by tk} and using the upper bound on $M_T$ \eqref{eq: app bound on mt}, we can write
\begin{align*}
\sum_{k=1}^{K_T}\frac{T_k}{\sqrt{t_k}} &\leq \sqrt{2T} \Big(7M_T\log\sqrt{2T}\Big) \\
&\leq 7\sqrt{2T}(1 + |\calS||\calA|\log (T+1)) \log\sqrt{2T}.
\end{align*}
\end{proof}

\section{Other Proofs}
\subsection{Proof of Lemma~\ref{lem: optimal policy}}
\textbf{Lemma} (restatement of Lemma~\ref{lem: optimal policy})\textbf{.} Suppose Assumption~\ref{ass: bellman equation} holds. Then, the policy $\pi^*(\cdot, \theta):\Delta_\calS \to \calA$ given by
\begin{equation}
\label{eq: the optimal policy}
\pi^*(b; \theta) := \argmin_{a \in \calA} \{c(b, a) + \sum_{o \in O}P(o | b, a; \theta)v(b'; \theta)\}
\end{equation}
is the optimal policy with $J_{\pi^*}(h; \theta) = J(\theta)$ for all $h \in \Delta_\calS$.
\begin{proof}
We prove that for any policy $\pi$, $J_\pi(h, \theta) \geq J_{\pi^*}(h, \theta) = J(\theta)$ for all $h \in \Delta_\calS$. Let $\pi:\Delta_\calS \to \calA$ be an arbitrary policy. We can write
\begin{align*}
&J_\pi(h, \theta) = \limsup_{T \to \infty} \frac{1}{T} \sum_{t=1}^T \E[C(s_t, \pi(h_t)) | s_1 \sim h] \\
&= \limsup_{T \to \infty} \frac{1}{T} \sum_{t=1}^T \E\Big[\E[C(s_t, \pi(h_t)) | \calF_t, s_1 \sim h] \big| s_1 \sim h\Big] \\
&=\limsup_{T \to \infty} \frac{1}{T} \sum_{t=1}^T \E[c(h_t, \pi(h_t)) | s_1 \sim h] \\
&\geq\limsup_{T \to \infty} \frac{1}{T} \sum_{t=1}^T \E[J(\theta) + v(h_t, \theta) - v(h_{t+1}, \theta) | s_1 \sim h] \\
&= J(\theta),
\end{align*}
with equality attained by $\pi^*$ completing the proof.
\end{proof}

\end{document}